\documentclass{article}

\PassOptionsToPackage{numbers,sort&compress}{natbib}
    \usepackage[preprint]{neurips_2026}

\usepackage[utf8]{inputenc} 
\usepackage[T1]{fontenc}    
\usepackage{hyperref}       
\usepackage{url}            
\usepackage{booktabs}       
\usepackage{amsfonts}       
\usepackage{nicefrac}       
\usepackage{microtype}      
\usepackage{xcolor}         
\usepackage{subcaption}
\usepackage{float}
\usepackage{enumitem}

\usepackage{amsmath}
\usepackage{amsthm}
\usepackage{graphicx}
\usepackage{tikz}

\newtheorem{theorem}{Theorem}[section]
\newtheorem{lemma}[theorem]{Lemma}

\newtheorem{assumption}{Assumption}

\title{Does Weight Decay Enhance Training Stability?}

\author{%
  Marius Saether\thanks{Equal Contribution. Author ordering determined by match-box flip over Zoom.}\\
  NTNU, MIT\\
  \texttt{masaet@mit.edu}\\
  \And
  Amir Kolic$^*$\\
  MIT\\
  \texttt{kolic@mit.edu}\\
  \And
  Tomaso Poggio \\
  MIT\\
  \texttt{tp@ai.mit.edu}\\
  \And
  Pierfrancesco Beneventano\\
  MIT\\
  \texttt{pierb@mit.edu}\\
}

\begin{document}

\maketitle

\begin{abstract}
In modern deep learning, weight decay is often credited with "stabilizing" training dynamics, diverging from its classical role as a static regularization penalty. We investigate a fundamental question: \emph{does weight decay stabilize training dynamics, and if so, through which mechanism?} Indeed, training stability is understood through different but related notions in the literature. We consider how weight decay affects the parameter-space dynamics and loss sharpness by analyzing its effects at the \emph{Edge of Stability} (EoS). We show that weight decay robustly slows \emph{progressive sharpening}. Furthermore, we uncover a striking architecture-dependent phase transition. In CNNs, weight decay dampens the oscillations at the EoS, while in MLPs, increasing weight decay causes a phase transition in which the sharpness stabilizes at a threshold significantly below the theoretical $\frac{2}{\eta}$ boundary. We develop a mathematical framework that accurately models these phenomena and identify the global alignment of the parameter vector and the sharpness gradient as the mechanistic driver of the phase transition. Importantly, we show that these phenomena translate into stability in terms of search in function-space (NTK). Last, this shows that curvature thresholds obtained from convex/quadratic heuristics may not be reliable stability diagnostics under regularization.
\end{abstract}

\section{Introduction}
\label{sec:introduction}
Rooted in classical Tikhonov regularization \citep{tikhonov} and the theory developed for ridge regression \citep{hoerl}, weight decay was historically viewed as a static penalty that restricts the hypothesis space and ensures well-posedness \citep{hanson, krogh}. Several papers have recently established the role of weight decay in picking specific features and solutions \citep{soudry, beneventano_implicit, jacobs, galanti, chen_lowrank, zangrando, yunis}. However, with the establishment of AdamW \citep{loshchilov} as the standard for pretraining, 
\begin{center}
    \textit{practitioners routinely credit weight decay for "stabilizing training"}.
\end{center}
In this paper, we investigate this hypothesis and the optimization folklore surrounding it.

Stability in neural networks generally refers to the stability of underlying dynamical systems. Three distinct interacting systems are relevant during training: \emph{the input propagation through the network}, \emph{the training process in function space}, and \emph{the training process in parameter space}. These correspond respectively to controlling the pre-activation scale (which led to the introduction of $\mu P$ \citep{mup}), bounding the Neural Tangent Kernel (NTK) \citep{ntk}, and bounding the loss curvature. 

We establish that \emph{weight decay heavily affects curvature through a strong phase transition mechanism}, an effect distinct from, and independent of, its role in feature selection as documented by prior literature. The mechanism operates by lowering the threshold around which the loss sharpness oscillates, which in turn governs the other two notions of stability described above.

More technically, a number of works showed that training neural networks happens at the \textit{Edge of Stability} (EoS): the relevant measure of curvature for the optimizer stabilizes around its own instability threshold \citep{xing_walk_2018,jastr_relation,jastrzebski,cohen,cohen_adaptive,andreyev_beneventano,andreyev_momentum,islamov}. For gradient descent (GD) with step size $\eta>0$ this resolves in the Hessian top eigenvalue $\lambda_{\max}\left(\nabla^2L\right)$ hovering at $\frac{2}{\eta}$ \citep{cohen}. Heuristically, one would expect that adding weight decay with parameter $\gamma > 0$ would impact $\lambda_{\max}\left(\nabla^2L\right)$ to hover at $\frac{2}{\eta} -\gamma$, and this was claimed in previous works. However, we show that \textit{above a certain critical $\bar \gamma > 0$, weight decay qualitatively changes progressive sharpening and level of stabilization.} For example, when $\gamma =0.02$ for an MLP trained on \texttt{cifar10-5k} subset with GD ($\eta = 2\times10^{-2}$) the stabilization level is $\lambda_{\max}\left(\nabla^2L\right)\approx 80$ instead of $\frac{2}{\eta}-\gamma=99.98$. Importantly, increasing $\gamma$ over $\bar \gamma$ doesn't harm test loss, and training often happens in that regime, thus this phase transition is benign. 

\begin{figure}[htbp]
    \centering
    \includegraphics[width=0.9\textwidth]{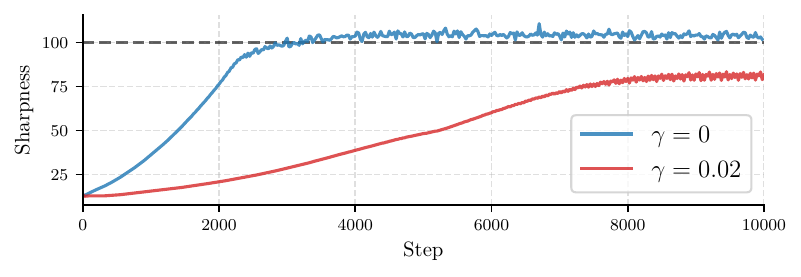}
    \caption{
    The evolution of the sharpness for varying weight decay values $\gamma$ on an MLP with ReLU activations trained with step size $\eta=0.02$. Weight decay \emph{reduces progressive sharpening strength} and \emph{shrinks the stabilization level} significantly below the expected $\frac{2}{\eta}-\gamma$ threshold.}
    \label{fig:intro_relu_phaseshift}
\end{figure}

We completely characterize this behavior mathematically and argue that this is due to the fact that \emph{weight decay introduces global interaction terms into the gradient update}, coupling the optimizer's local geometry to the global position of the parameter vector.

Specifically, our contributions are as follows:
\begin{itemize}[leftmargin=1em,itemsep=0.6em,topsep=0.15em,parsep=0pt]
    \item \textbf{Training still happens at the EoS.} We establish that training with weight decay continues to operate at the Edge of Stability. However, we clarify that this occurs not in the classical sense, but rather as defined by Andreyev and Beneventano \citep{andreyev_beneventano}, operating continuously at the boundary of local instability. Crucially, we demonstrate that observing a maximum eigenvalue significantly below the naive regularized threshold ($\lambda_{max} \ll \frac{2}{\eta} - \gamma$) does not imply that the training dynamics have become intrinsically stable.
    
    \item \textbf{Characterization of the effects of weight decay on Edge of Stability dynamics.} We empirically demonstrate that weight decay consistently slows progressive sharpening, and delays the onset of EoS. In CNNs, it also heavily dampen sharpness oscillations. Most notably, in MLPs, we identify a phase transition in which increasing weight decay beyond a critical threshold causes the stabilization level to shrink significantly below the expected theoretical boundary.
    
    \item \textbf{Introducing a model of weight decay at the EoS that captures empirical observations.} Building upon the self-stabilization framework of \citep{damian_selfstab}, we model gradient descent with weight decay as an underdamped harmonic oscillator. By incorporating a global alignment term, our model accurately predicts the observed damping and limit cycle collapse. This framework mathematically argues that the slowdown in progressive sharpening and the MLP phase transition occur precisely because weight decay introduces non-local dynamics, forcing global interactions to dictate the evolution of the local geometry.
    
    \item \textbf{This, however, stabilizes the dynamics in function space.} Building on our dynamical model, we show that the depressed sharpness threshold induced by weight decay has direct implications for function-space stability. Specifically, we establish a structural link demonstrating that this lower stabilization threshold restricts the spectrum of the empirical Neural Tangent Kernel (NTK). 
    \item \textbf{Curvature thresholds may not be reliable stability diagnostics.}
    We show that weight decay can drive the observed sharpness far below the naive regularized threshold $\frac{2}{\eta}-\gamma$ while training still exhibits characteristic EoS behavior.
    This implies that the stability cannot always be diagnosed by curvature/sharpness being lower than a threshold given by quadratics/convex heuristics.
\end{itemize}

\begin{figure}[htbp]
    \centering
    \includegraphics[width=0.8\textwidth]{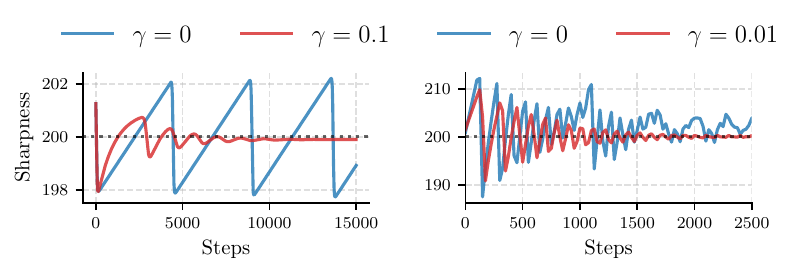}
    
    \caption{Weight decay \emph{dampens the oscillations} at the EoS. On the left, the dampening for a toy loss model along with a visible $\gamma=0.1$ shift in the stabilization threshold. On the right, a CNN trained on a \texttt{cifar10-5k} subset, showing the behavior of a dampened harmonic oscillator for $\gamma=0.01$. Both had a learning rate of $\eta=0.01$, and the dashed line is at $\frac{2}{\eta}$. The chaotic behavior of $\gamma=0$ after step $\approx 1000$ for the CNN is due to multiple eigenvalues at the EoS (Appendix \ref{app:multi_eigenvalues}).}
    \label{fig:intro_dampening}
\end{figure}

\section{Related Work}
\label{sec:related_work}
\paragraph{Effects of Weight Decay} Classical work frames weight decay as a static penalty for ill-posed inverse problems \citep{tikhonov,hoerl}. In neural networks, it restricts the hypothesis space, penalizes model complexity, and ensures well-posedness to improve generalization \citep{hanson, krogh}. However, in modern deep learning, weight decay is used primarily as a trajectory stabilizer \citep{dangelo, laarhoven}. For example, \citep{xie} introduces a weight decay scheduler that avoids large gradient-norms, while \citep{yun} explores weight decay scheduling and its effects on model performance. \citep{golatkar, bjorck} highlight the importance of weight decay in the early phases of training. \citep{kosson_rotational} shows that weight decay causes the magnitude and angular updates of weight vectors to converge to a steady state, improving stability. Our paper expands on the literature exploring weight decay as a dynamic stabilizer by analyzing its effect on loss sharpness. We further demonstrate how these sharpness effects translate into broader notions of training stability.

\paragraph{Edge of Stability}
As previously noted, neural network training typically enters the Edge of Stability (EoS) following a phase of \textit{progressive sharpening} \citep{jastrzebski, cohen}. Once in this highly oscillatory regime, local self-stabilizing dynamics prevent the optimizer from completely diverging \citep{damian_selfstab}, a macroscopic trajectory recently formalized via the central flow framework \citep{central_flow}. While extensive literature has mapped how various optimization choices such as stochasticity, momentum, and non-Euclidean geometries perturb these stability boundaries \citep{andreyev_beneventano, andreyev_momentum, islamov}, the specific interaction with weight decay remains underexplored. Most relevantly, \citep{lyu} shows that in scale-invariant networks weight decay balances parameter norm increases and lowers the \emph{spherical sharpness} of the loss, characterizing the trajectory via a continuous sharpness-reduction flow. Our work studies the effect of weight decay on EoS in general architectures, developing a mathematical model that accounts for all of our observed empirical phenomena.

\paragraph{Notions of Stability}
As Figure \ref{fig:big_picture_story} summarizes, training stability is explored through three intertwined lenses: the parameter-space landscape (EoS), function-space dynamics (NTK), and signal propagation scaling ($\mu$P). While the NTK \citep{ntk} governs prediction dynamics and active feature learning, $\mu$P ensures $\mathcal{O}(1)$ stability for activations and gradients across widths \citep{mup}. Recent works deeply connect these regimes: progressive sharpening to the EoS is intrinsically tied to the NTK's largest eigenvalue, with subsequent EoS sharpness reductions driving NTK target alignment \citep{noci, ntk_eos, lauditi2026spectral}, this links local sharpness reductions to feature alignment. Furthermore, \citep{wdmup} finds that when $\mu$P's alignment assumptions fail early in practice, independent weight decay takes over to stabilize $\mathcal{O}(1)$ updates and facilitate learning rate transfer. By analyzing weight decay at the EoS, our work unifies these perspectives, mathematically modeling how weight decay modulates sharpness, NTK evolution, and broader stability.
\begin{figure}[htbp]
\centering
\begin{tikzpicture}
  \node (EoS) at (0, 2.4) {EoS};
  \node (NTK) at (6, 2.4) {NTK};
  \node (WD) at (0, 0) {WD};
  \node (muP) at (6, 0) {$\mu$P};

  \draw[<->] (EoS) -- node[above=2pt] {\citep{ntk_eos, noci}} (NTK);
  \draw[<->] (NTK) -- node[right=6pt] {\citep{mup}} (muP);
  \draw[->, color=blue] (WD) -- node[left=10pt] {\textcolor{blue}{Our work}} (EoS);
  \draw[->, color=blue] (WD) --  (NTK);
  \draw[<-] (muP) -- node[below=2pt] {\citep{wdmup}} (WD);
\end{tikzpicture}
\caption{A diagram illustrating three intertwined notions of training stability and weight decay's role as a trajectory stabilizer.}
\label{fig:big_picture_story}
\end{figure}
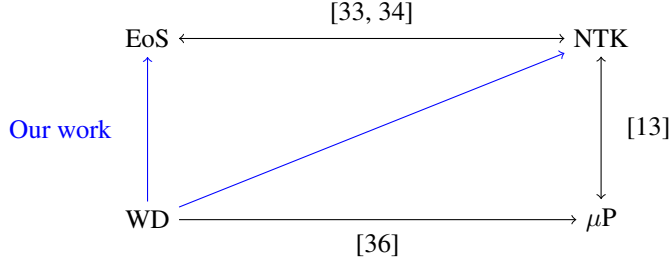

\section{Weight Decay Empirically Changes EoS Dynamics}
\label{sec:empirical}

The effect of weight decay at the Edge of Stability (EoS) can naively be estimated by considering the regularized loss $\tilde{L}=L+\frac{\gamma}{2}||\theta||^2$. Since the effective sharpness is increased by $\gamma$, one might expect the EoS threshold to simply shift down to $\frac{2}{\eta}-\gamma$. However, we observe more complex dynamics that depend non-linearly on the weight decay. In this section, we empirically establish three key phenomena: \textbf{1.} Weight decay robustly slows down progressive sharpening across all models. \textbf{2.} Once at the EoS, weight decay induces architecture-dependent behaviors. Specifically, it dampens the sharpness oscillations in CNNs, while in MLPs it triggers a phase transition to a dramatically lower stabilizing sharpness. \textbf{3.} Crucially, this does not mean weight decay "stabilizes" training by avoiding the EoS. Even when the phase transition shifts the sharpness to a much lower baseline, \textit{the network continues to train in the EoS regime} (e.g., exhibiting non-monotonic loss decreases and persistent gradient oscillations). Thus, \emph{weight decay does not stabilize the trajectory from a purely curvature-based perspective}, but rather modulates and lowers the amplitude of sharpness oscillations.

\paragraph{Progressive Sharpening Slowdown.} 
Across architectures, we find that weight decay fundamentally slows the rate of progressive sharpening. As illustrated in Figure \ref{fig:empirical_combined}, increasing $\gamma$ significantly delays the onset of the EoS and visibly flattens the sharpness trajectory during the early phases of training for both MLPs and CNNs.

\paragraph{Dampening Oscillations vs. Phase Transitions.}
Once the network enters the EoS, the effects of weight decay diverge based on the architecture. In CNNs, weight decay acts primarily to dampen oscillations. As shown in Figure \ref{fig:empirical_combined}, for larger values of $\gamma$, the sharpness stabilizes faster and with significantly reduced oscillatory amplitude after reaching the $\frac{2}{\eta}$ threshold.

In MLPs with ReLU activations trained with full-batch gradient descent, weight decay induces a distinct phase transition in the EoS dynamics. Below a critical threshold ($\gamma < \bar{\gamma}$), the stabilizing sharpness remains largely unaffected. However, for $\gamma > \bar{\gamma}$, the sharpness stabilizes at progressively lower values, well below the conventional $\frac{2}{\eta}-\gamma$ boundary. Importantly, despite this dramatically lowered threshold, the dynamics retain characteristic EoS signatures, including sharpness oscillations and non-monotonic loss steps (Appendix \ref{EoS_MLP_osc}).

\begin{figure}
    \centering
    \includegraphics[width=\textwidth]{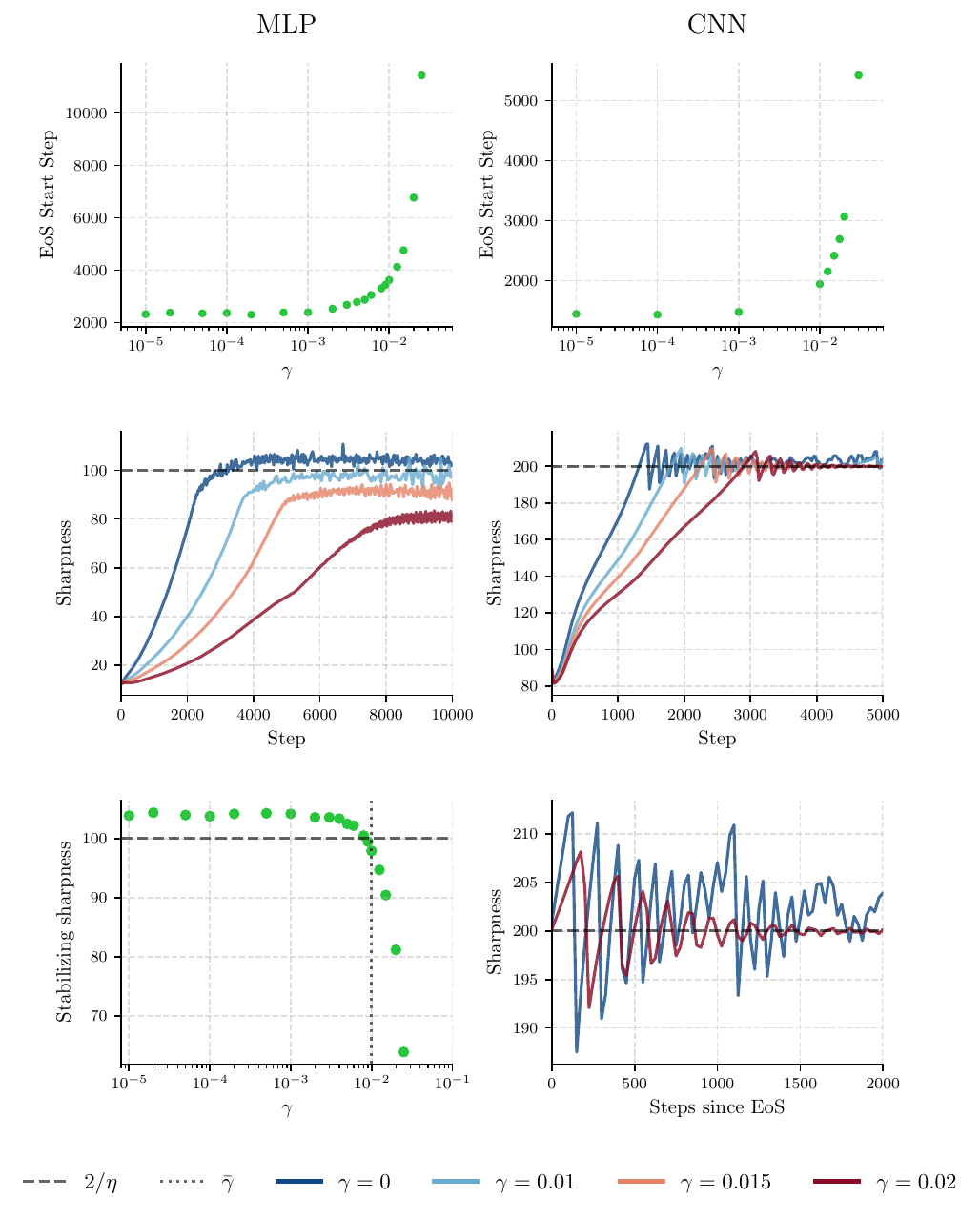}
    \caption{Empirical evaluation of weight decay's effect on Edge of Stability (EoS) dynamics for MLPs (left column, $\eta=0.02$) and CNNs (right column, $\eta=0.01$). The dashed black line represents the theoretical $\frac{2}{\eta}$ threshold. \textbf{Top Row:} The EoS onset (first step of non-monotonic loss) as a function of weight decay $\gamma$. Increasing $\gamma$ consistently delays entry into the EoS across both architectures. \textbf{Middle Row:} The evolution of loss sharpness over training steps for varying $\gamma$. The trajectories visually demonstrate the slowdown of progressive sharpening, as well as the oscillatory dynamics of the EoS regime and where the stabilization threshold is. \textbf{Bottom Row:} Architecture-dependent behaviors. For MLPs (left), the stabilizing sharpness is plotted against $\gamma$, revealing a phase transition where sharpness stabilizes significantly below $\frac{2}{\eta}-\gamma$ for $\gamma > \bar{\gamma}$. For CNNs (right), comparing the dynamics immediately following EoS entry shows that weight decay $\gamma=0.02$ strongly dampens the oscillations at the EoS compared to those observed in training with $\gamma=0$.}
    \label{fig:empirical_combined}
\end{figure}

\section{The Mechanism is Global Interaction}
\label{sec:model}

Here we model how a weight decay of $\gamma$ affects the EoS phenomena. Our analysis shows that the sharpness at the EoS can be modeled as a dampened harmonic oscillator, with weight decay being the dampening factor. Notably, this single analytical framework is able to account for practically all of the empirical phenomena observed in Section \ref{sec:empirical}. While the behavior of optimizers is often analyzed locally, we show that the striking phase shift behavior that weight decay induces at the EoS is caused by a global interaction term, related to the parameter vector. In particular, we demonstrate that when the global interaction term $\gamma\langle \theta, \nabla S\rangle$ crosses a specific critical threshold, it causes a phase shift and the sharpness stabilizes below the naive $\frac{2}{\eta}-\gamma$ threshold. This is supported by empirical evidence measuring alignment with the $\nabla S$ vector in varying architectures. The full derivation is given in Appendix \ref{app:model}.

\subsection{Local Friction and the Underdamped Oscillator}

Let the sharpness at $\theta$ be $S(\theta)=\lambda_{\max}\left(\nabla^2L(\theta)\right)$ and let $u(\theta)$ be the corresponding eigenvector $\nabla^2L(\theta)u(\theta)=S(\theta)u(\theta)$. To understand the dynamics, we analyze gradient descent around a fixed $\theta^*$ in the local basis spanned by the maximum sharpness direction $x_t = u \cdot (\theta_t - \theta^*)$ and the progressive sharpening direction $y_t = \nabla S \cdot (\theta_t - \theta^*)$, where we let $S=S(\theta^*),u=u(\theta^*)$. We assume that throughout training, the \emph{progressive sharpening} coefficient $\alpha=-\langle\nabla L,\nabla S\rangle$ satisfies $\alpha>0$, i.e., there is a progressive sharpening force. We assume that the dynamics are governed by a single top eigenvalue at the EoS (Appendix \ref{app:multi_eigenvalues}). Without weight decay, the optimization at the EoS forms a stable negative feedback loop. The sharpness grows due to progressive sharpening, but is forcefully bounded when the unstable bounces in $x_t$ become large enough to trigger third-order stabilization effects \citep{damian_selfstab}.

When a weight decay penalty of $\gamma$ is introduced, the update rule naturally spawns local and global friction terms. The dynamical system governing the discrete steps can be simplified to:
\begin{align}
    x_{t+1} &\approx -x_t(1 + \eta y_t + \eta\gamma) \label{eq:x_friction} -\eta\gamma c_x \\
    y_{t+1} &\approx y_t(1 - \eta\gamma) + \eta\left(\alpha - \beta \frac{x_t^2}{2}\right) -\eta\gamma c_y\label{eq:y_friction}
\end{align}
with the global terms $c_x=u\cdot\theta^*$ and $c_y=\nabla S\cdot\theta^*$. We can consider only the local dynamics and assume that $c_x\approx0, c_y\approx0$. This yields a dynamical system that we analyze in \ref{app:model}. The solution to the linearized system shows how the weight decay acts as a linear damper, with sharpness oscillations decaying exponentially with respect to $\gamma$.

\begin{figure}[htbp]
    \centering
    \includegraphics[width=0.8\textwidth]{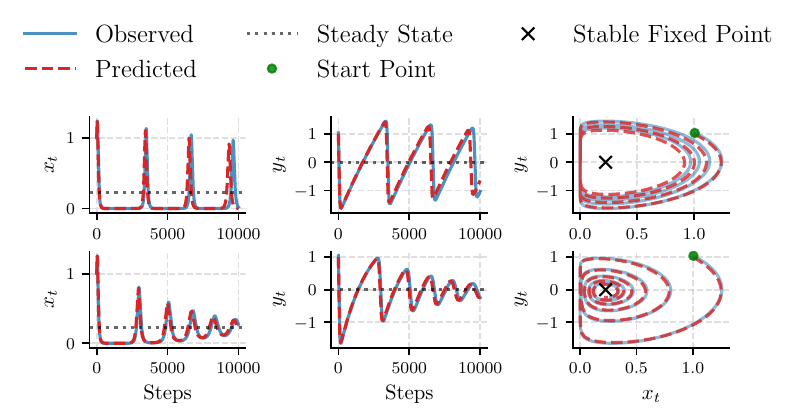}
    \caption{The model predicts the dynamics of $x_t$ and $y_t$ on the left and in the middle column, respectively. The $x_t,y_t$ phase space is shown on the right. Increasing the weight decay introduces a stronger dampening factor, and we see that the sharpness oscillations decay faster. The $y_t$ resting threshold is also shifted by $-\gamma$.}
    \label{fig:model_accuracy}
\end{figure}

This formulation exactly predicts two of our empirical observations. Firstly, weight decay \emph{dampens the oscillations at the EoS}, which decay exponentially with a factor proportional to $\gamma$. This matches the highly dampened, settling oscillations observed in CNNs (Figure \ref{fig:empirical_combined}). Secondly, below the phase shift boundary $\gamma<\bar{\gamma}$, the stabilization sharpness is shifted by a factor $\gamma$ to $S(\theta) = \frac{2}{\eta} - \gamma$. 

\subsection{Global Offsets and Limit Cycle Collapse}

While the local friction explains the damped oscillations, it does not explain the drastic phase transition where we observe EoS at a much lower threshold (Figure \ref{fig:empirical_combined}). To account for this, we include the global decay offsets $c_x$ and $c_y$. While the system can harmlessly absorb $c_x$ by slightly "shifting the center" of its bounces along the valley of the trajectory, $c_y$ fundamentally alters the underlying sharpening force and predicts phase transitions for a large enough interaction term.

\begin{figure}[htbp]
    \centering
    \includegraphics[width=0.6\linewidth]{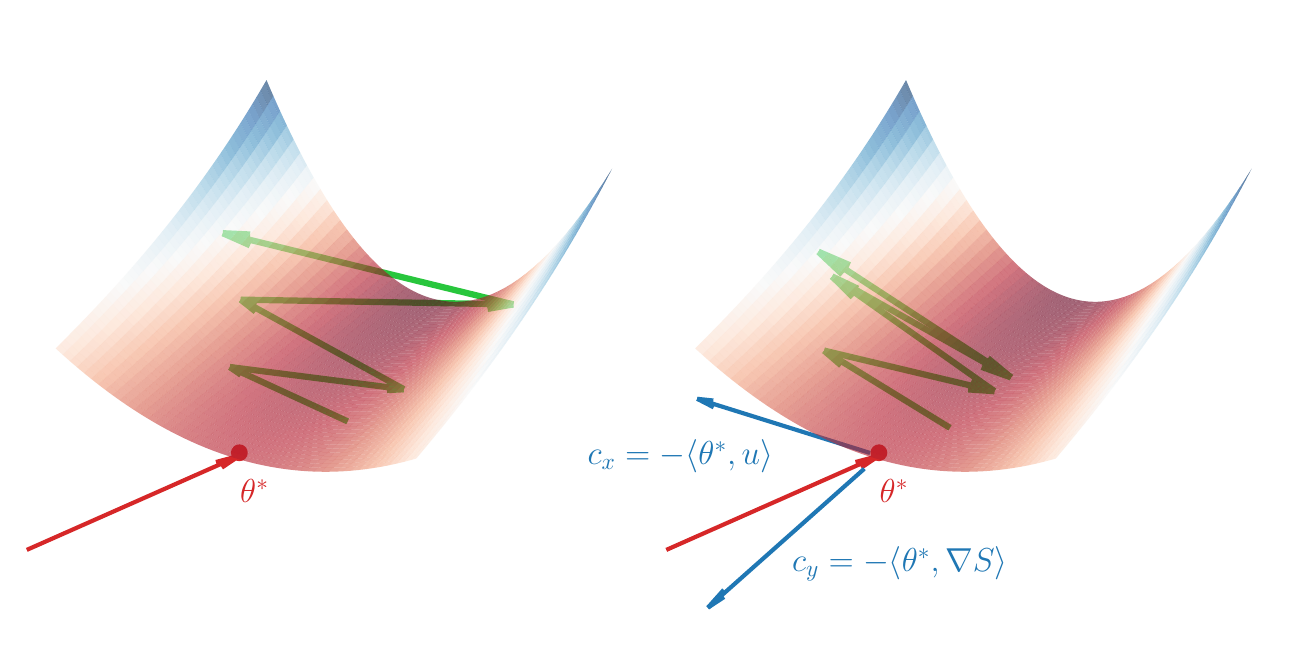}
    \caption{An illustration of a simplified mental model of EoS introduces optimization along a sharpening valley. The $c_x$ interaction term shifts the balance of the $x_t$ oscillations slightly in one direction, while $c_y$ fundamentally changes where the sharpness oscillates.}
    \label{fig:model_simplified_global}
\end{figure}

By averaging the sharpness update over a full two-step oscillation cycle, we find that the mean of the $x_t$ bounces is given by $-\frac{\eta\gamma c_x}{2}$, and resting amplitude of the unstable bounces, $\Delta$, is given by:
\begin{equation}
    \Delta = \sqrt{\frac{2(\alpha + \gamma^2 - \gamma c_y)}{\beta}}
\end{equation}

If the global offset $c_y$ is large enough, the amplitude squared ($\Delta^2$) would be negative. Therefore, this system possesses a hard critical bifurcation threshold:
\begin{equation}
    c_y^{\text{crit}} = \frac{\alpha}{\gamma} + \gamma
\end{equation}

While this local derivation assumes a fixed reference point $\theta^*$ to isolate the oscillator mechanics, the true optimization trajectory drifts continuously along the loss manifold. Consequently, the global interaction term itself evolves. If the network's trajectory causes $c_y$ to cross the critical threshold $c^{\text{crit}}_{y}$ during progressive sharpening, the stabilizing sharpness plunges much lower than $\frac{2}{\eta}-\gamma$.

This theoretical collapse is supported by the empirical phase transitions observed. In Figure \ref{fig:model_phaseshift} we see that for MLPs trained with sufficient weight decay, the global offset $c_y$ term crosses the $c_y^{\text{crit}}$ threshold before the EoS, forcing the network to stabilize at a heavily reduced curvature. For CNNs, however, the global interaction term remains below the critical threshold during the entire progressive sharpening phase.

\begin{figure}[htbp]
    \centering
    \includegraphics[width=0.8\textwidth]{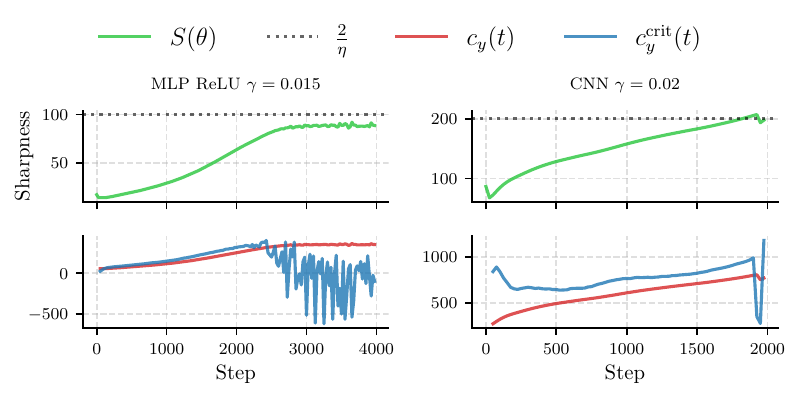}
    \caption{The sharpness dynamics (top) and the evolution of the $c_y$ global interaction term compared to the progressive sharpening coefficient dependent critical threshold $c^{\text{crit}}_y$ during training (bottom). We compare an MLP ReLU with $\eta=0.02,\gamma=0.015$ (left) and a CNN with $\eta=0.01,\gamma=0.02$. The first few steps in the $c_y$ plots are ignored for plotting as $\alpha\ll0$ (note the initial sharpness drop).}
    \label{fig:model_phaseshift}
\end{figure}

\section{How Weight Decay and EoS Stabilize the Training}
\label{sec:ntk}

Our preceding analysis demonstrates that while weight decay does not prevent training at the Edge of Stability (EoS), the phase transition it induces, lowering the sharpness threshold significantly, naturally translates into improved function-space stability. Without regularization, unconstrained parameter growth yields excessively sharp functional representations. Bounding the NTK spectrum may therefore contribute to stable signal propagation (the core objective of frameworks like $\mu$P) and robust feature learning.

Let $f(\theta, X) \in \mathbb{R}^N$ be the network outputs over an empirical batch of size $N$. The function-space dynamics are governed by the empirical NTK, $\Theta = J J^T \in \mathbb{R}^{N \times N}$, where $J \in \mathbb{R}^{N \times P}$ is the Jacobian. Under MSE loss, the Hessian $H$ decomposes exactly as:
\begin{equation}
    H = J^T J + \sum_{i=1}^N (f(\theta, x_i) - y_i)\nabla^2 f(\theta, x_i) = J^TJ + R
\end{equation}
Because $J^TJ$ and $\Theta$ share the same non-zero spectrum, Weyl's inequality bounds the distance between their largest eigenvalues by the spectral norm of the residual $R$:
\begin{equation}
\label{eq:ntk_eos_bound}
    |\lambda_{\max}(H) - \lambda_{\max}(\Theta)| \le \|R\|_2
\end{equation}
As the network minimizes the empirical risk, $\|R\|_2$ approaches zero. In this regime, the Hessian's decreased EoS threshold tightly bounds the NTK's spectral radius. However, the benefit of a lower sharpness threshold extends beyond simply capping this radius; it \emph{actively improves feature learning}. \citep{ntk_eos} demonstrates that periods of sharpness reduction during EoS are fundamentally linked to an increase in Kernel Target Alignment (KTA). Assuming an approximate rank-1 NTK structure during training, $K_t \approx c_t^2 X^T X + \|v_t\|^2 \hat{y}\hat{y}^T$ (where $c_t^2$ scales sharpness and $v_t$ tracks network output), we extract the following lemma:

\begin{lemma}[Adapted from Theorem 1 and Lemma 3 of \citep{ntk_eos}]
Let $\alpha_t := \frac{\|v_t\|^2}{c_t^2}$ be the ratio between the output norm component and the sharpness component of the NTK. Under mild regularity conditions on the residual, a reduction in the sharpness threshold yields an overall increase in $\alpha_t$. Furthermore, for an appropriate range, increasing $\alpha_t$ shifts the maximum alignment of the target vector $Y$ towards the leading eigenvectors of the NTK.
\label{lemma:ntk_connection}
\end{lemma}
\textit{The formal statement and proof of this lemma are deferred to Appendix \ref{app:ntk_connection}.} This lemma provides a rigorous heuristic for the stabilizing effect of weight decay. By triggering a phase transition that stabilizes sharpness at a significantly lower threshold, weight decay restricts $c_t^2$. This mechanically inflates the ratio $\alpha_t$, forcing the NTK's leading eigenvectors to align much more strongly with the target vector. Ultimately, weight decay acts as a dynamic limit on the complexity of the learned functional representation, ensuring robust feature alignment. While fully developing the connection between the EoS and NTK alignment dynamics is beyond the scope of this work (we refer readers to \citep{ntk_eos} for a detailed exploration), Figure \ref{fig:ntk_eigenvalue_decrease} corroborates these theoretical insights by empirically demonstrating how the top eigenvalue of the NTK systematically decreases as the weight decay parameter $\gamma$ increases.

\begin{figure}[htbp]
    \centering
    \includegraphics[width=0.8\textwidth]{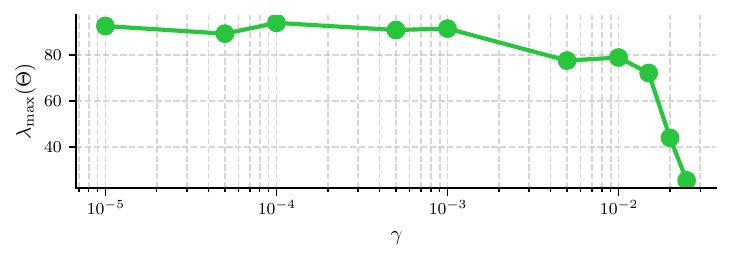} 
    \caption{The top eigenvalue of the empirical NTK for an MLP, $\lambda_{max}(\Theta)$, evaluated after sharpness stabilization, as a function of weight decay $\gamma$, with $\eta=0.02$. As predicted by Equation \ref{eq:ntk_eos_bound}, the lowered parameter-space EoS threshold explicitly restricts the function-space NTK spectrum.}
    \label{fig:ntk_eigenvalue_decrease}
\end{figure}

\section{Implications}
\label{sec:implications}

Our findings support that weight decay acts as an active and dynamic participant in the training trajectory. Modeling the sharpness dynamics at the EoS reveals that local dynamics are heavily influenced by global interactions. Specifically, the global parameter alignment $c_y = \langle \nabla S, \theta^* \rangle$ strongly dictates the bounds of local curvature where phase shifts occur. This framework yields three primary implications for deep learning:
\begin{itemize}[leftmargin=1em,itemsep=0.6em,topsep=0.15em,parsep=0pt]
    \item \textbf{Revisiting Curvature as a Stability Proxy:} We demonstrate that low loss sharpness does not necessarily imply a departure from the EoS. In MLPs, even when the maximum eigenvalue drops significantly below the regularized threshold ($\lambda_{max} \ll \frac{2}{\eta} - \gamma$), the model continues to exhibit EoS signatures. Consequently, researchers must look beyond purely curvature-based thresholds to accurately diagnose training stability.
    
    \item \textbf{Weight Decay improves Function-Space Stability \textit{in a precise sense}:} Recent literature links EoS sharpness reductions to improved Neural Tangent Kernel (NTK) target alignment \citep{ntk_eos}. Our framework provides the mechanism for how weight decay interplays with this phenomenon. With a phase transition that stabilizes sharpness at a heavily reduced threshold, weight decay actively restricts the NTK spectrum, yielding function-space training improvements and stability.
    
    \item \textbf{Weight Decay as a Multi-Stability Control Signal:} Classical optimizer design treats weight decay as a scalar regularization hyperparameter, tuned independently of learning rate scheduling or architecture. Our results complicate this picture because a single choice of $\gamma$ simultaneously governs the EoS sharpness threshold, and in turn the NTK spectral radius, its consequences are entangled across all different notions of stability. Crucially, this effect is architecture-dependent. Optimizer designers should therefore treat weight decay not as a static penalty, but as a geometry-aware control signal whose effect is jointly determined by architecture, learning rate, and the evolving global alignment $c_y = \langle \nabla S, \theta \rangle$.

\end{itemize}
\section{Limitations}
\label{sec:limitations}

We acknowledge and highlight two primary limitations of our work.

\textbf{Scale of experiments.} Our empirical evaluation is limited to small architectures and datasets, as computing the top Hessian eigenvalue and the third-order quantities in our model costs $O(d \cdot n)$ per step via Hessian-vector products. This constraint is shared by the broader EoS literature, including the foundational studies \cite{cohen, andreyev_beneventano} from which we adapted our setup. Thus while the limitations are strong, they are motivated and shared by the rest of the relevant literature.

\textbf{Full-batch gradient descent.} Our analysis is further restricted to full-batch gradient descent, as the self-stabilization framework \cite{damian_selfstab} on which our model builds has not been successfully extended to mini-batch optimizers, a gap widely recognized as a major open problem in this field \cite{andreyev_beneventano, damian_selfstab}. Extending our framework to mini-batch and adaptive settings, developing tractable sharpness proxies for large-scale models, and developing self-stabilization mechanisms for mini-batch optimizers, are natural directions for future work.

\section{Conclusion}

This paper reveals that weight decay affects the training dynamics at the Edge of Stability (EoS), surprisingly qualitatively differently than how expected from naive arguments. By modeling gradient descent with weight decay as an underdamped harmonic oscillator, we develop a highly predictive framework capturing both local friction and global interactions. Our model demonstrates that local weight decay friction dampens sharpness oscillations, as observed in CNNs, while a sufficiently large global alignment between the parameter vector and sharpness gradient triggers a striking phase transition in MLPs which reduce the sharpness threshold. Crucially, we show that this depressed parameter-space threshold explicitly restricts the spectrum of the empirical Neural Tangent Kernel (NTK). \textit{This is the first result, to the knowledge of the authors, supporting the practitioners' folklore that weight decay enhances stability}. Ultimately, weight decay acts as a dynamic limit on representational complexity, bridging parameter-space stabilization with robust function-space feature alignment.

\newpage

\bibliographystyle{unsrt}
\bibliography{references}

\newpage
\appendix

\section{Empirical Results}
\label{app:empirical}

\subsection{EoS behaviour at lower sharpness threshold}\label{EoS_MLP_osc}
Figure \ref{fig:in_fact_eos} shows an MLP trained with stepsize $\eta=0.02$ and weight decay $\gamma=0.02$. The sharpness stabilizes around 80, far below the weight decay adjusted EoS threshold $\frac{2}{\eta}-\gamma=99.98$. We argue that the model in fact operates at the \emph{Edge of stability}, despite the sharpness being significantly lower than the conventional threshold. 

In early training, during progressive sharpening, the loss decreases monotonically, shown by the \emph{per step difference in loss} $\Delta L$ being strictly negative. Around step 7000, however, the sharpness stabilizes at approximately 80, and the dynamics change qualitatively. $\Delta L$ begins alternating in sign, with individual steps frequently increasing the loss, directly contradicting the descent lemma, yet the system does not diverge. The continued downward trend in the loss demonstrates that the time-averaged per-step loss change is negative throughout the oscillations. This confirms that the training is not in a descent lemma scenario, but operates in a barely stable regime. The model thus exhibits behavior consistent with EoS, despite $\lambda_{max}\ll\frac{2}{\eta}-\gamma$.

\begin{figure}[h!]
    \centering
    \includegraphics[width=0.66\linewidth]{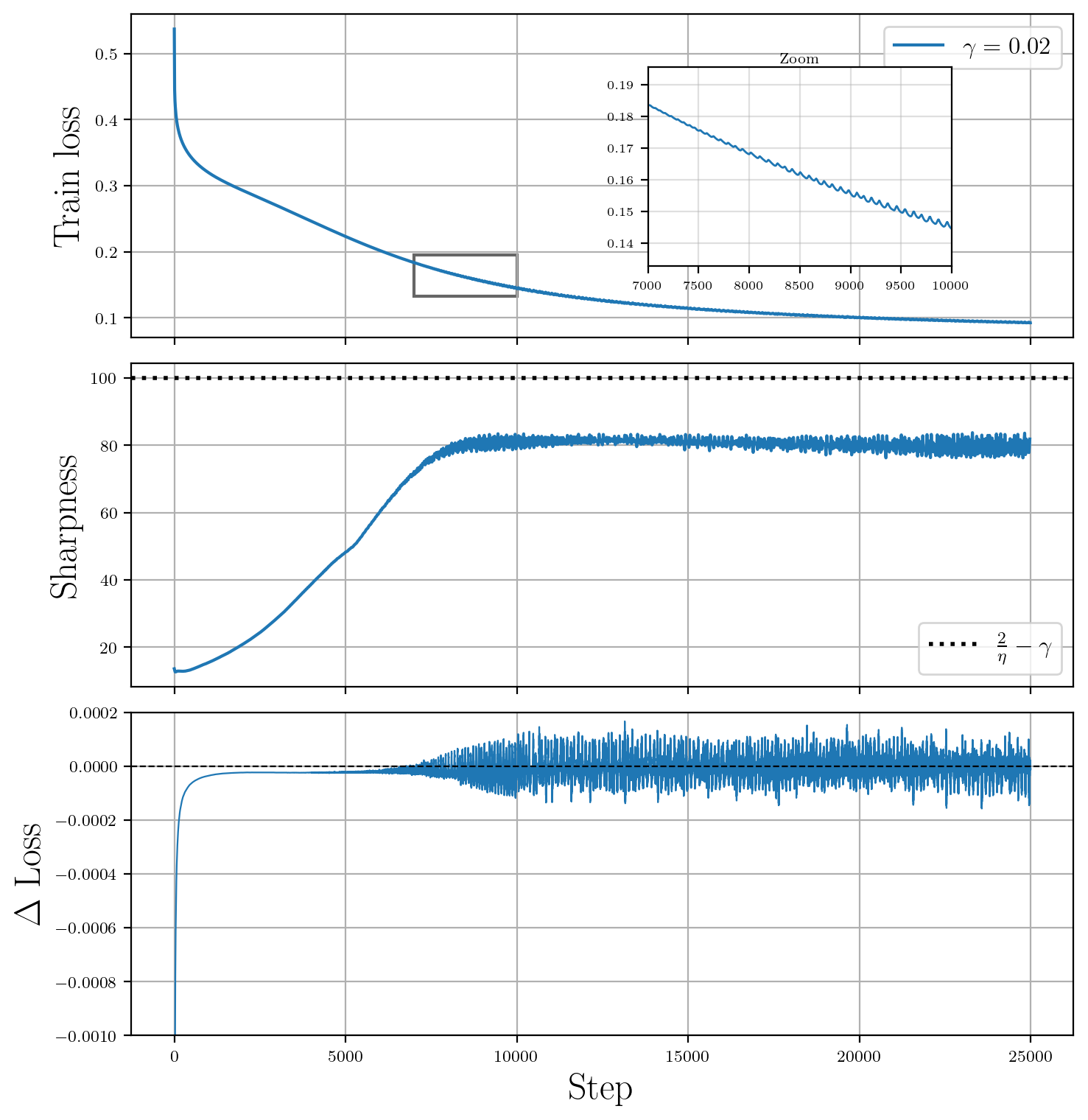}
    \caption{Training dynamics for an MLP with ReLU activations trained with full batch gradient descent and MSE loss on a 5000 sample subset of \texttt{cifar10}, with weight decay $\gamma=0.02$, and learning rate $\eta=0.02$. The sharpness $(\lambda_{max})$ stabilizes at approximately 80, far below the predicted EoS threshold of $\frac{2}{\eta}-\gamma=99.98$. Once the sharpness stabilizes, the per-step loss difference $\Delta L$ oscillates, while the overall loss continues to decrease. The y-axis of the $\Delta L$ panel is clipped to highlight the oscillations. The initial per-step loss decrease is substantially larger in magnitude.}
    \label{fig:in_fact_eos}
\end{figure}

\newpage

\subsection{On the choice of learning rate}
We use $\eta=0.02$ for MLP and $\eta=0.01$ for CNN. The difference is motivated by the different architectures' initialization sharpness. As seen in Figure $\ref{fig:PS_both_arcs}$, the MLP initializes at sharpness $\sim10$, far below the EoS threshold of $2/\eta=100$, whereas the CNN initializes at sharpness $\sim90$, already close to the $\eta=0.02$ threshold of 100. By choosing $\eta$ such that the initialized model sits well below $2/\eta$ for both architectures, we ensure a comparable pre-EoS starting point, that makes the effects of weight decay easier to isolate and compare.

\begin{figure}[htbp]
    \centering
    \includegraphics[width=1\linewidth]{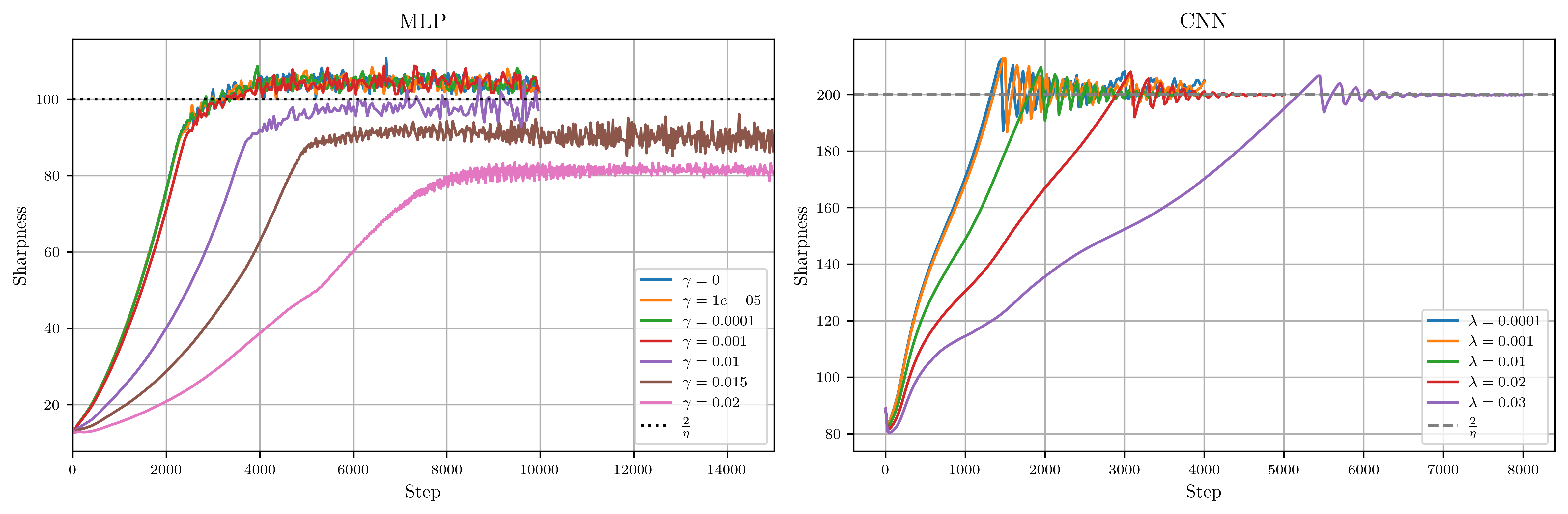}
    \caption{Sharpness trajectories for an MLP (left, $\eta=0.02$) and a CNN (right, $\eta=0.01$), trained with full batch GD on a 5k subset of \texttt{cifar10}. Increasing $\gamma$ slows progressive sharpening, delaying EoS onset. For large $\gamma$, sharpness stabilizes below $2/\eta-\gamma$ in MLP (left). Further analyzed in Appendix \ref{Sharpness threshold}}
    \label{fig:PS_both_arcs}
\end{figure}

\subsection{EoS threshold in MLPs}
\label{Sharpness threshold}
We empirically investigate how weight decay affects the value at which the sharpness $\lambda_{max}(\nabla^2L(\theta))$ stabilizes during training. We train a 2-layer MLP with 200 hidden units and ReLU activations on a 5000-sample subset of \texttt{cifar10}, using full batch gradient descent with learning rate $\eta=0.02$. We sweep weight decay $\gamma\in [10^{-5}, 3\times 10^{-2}]$ across 20 independent runs, with all other hyperparameters fixed. Sharpness is estimated during training via power iteration, and logged every 25 steps. Note that runs with larger $\gamma$ require significantly more steps to reach stabilization due to the slowdown in progressive sharpening. For each run, the stabilizing value is estimated as the mean of $\lambda_{max}$ over the final $10\%$. As shown in Figure \ref{fig:PS_both_arcs}, all the represented runs reach a stable regime well before the end of training. 

For values $\gamma>2.5\times 10^{-2}$, progressive sharpening slows down to the extent that we were unable to reach stabilization within a reasonable training horizon. Preliminary runs suggest that the sharpness continues to stabilize at progressively lower values. 


\begin{figure}[htbp]
    \centering
    \includegraphics[width=0.8\linewidth]{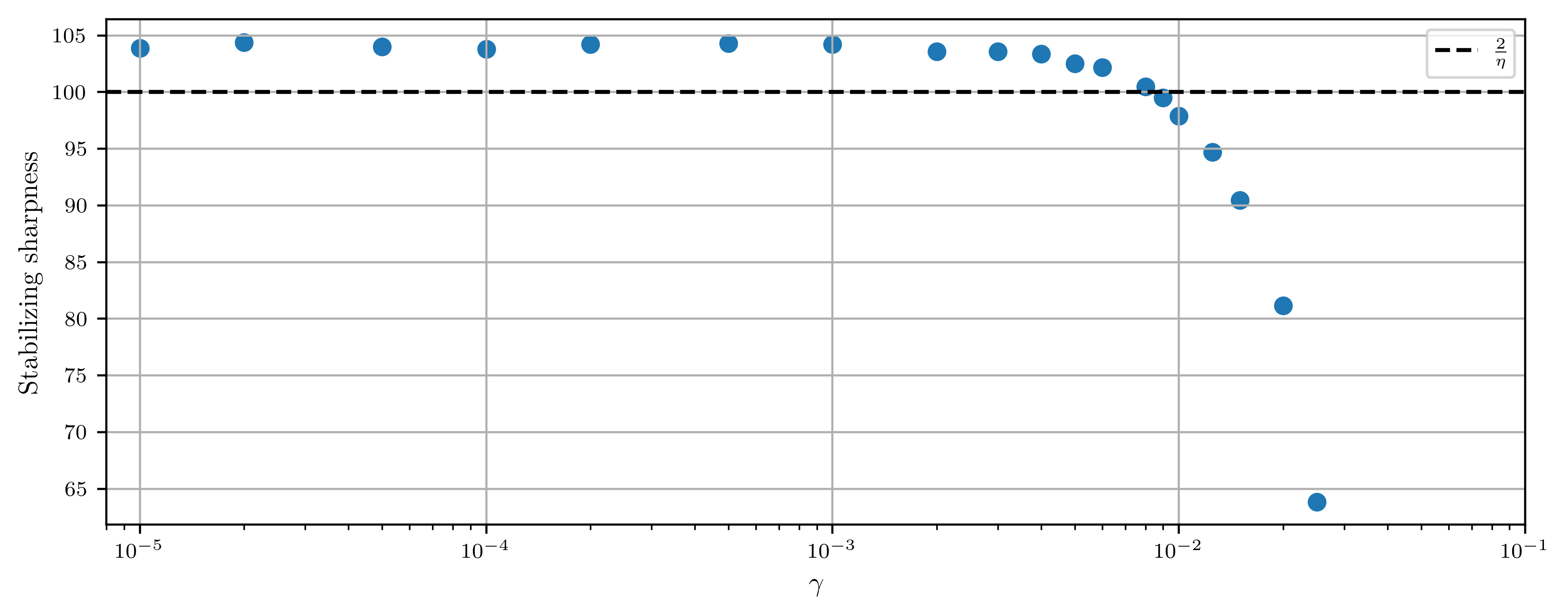}
    \caption{Stabilizing sharpness as a function of $\gamma$ (reproduced from Figure \ref{fig:empirical_combined})}
    \label{fig:phaseshift2}
\end{figure}

\newpage

\subsection{Chaotic oscillations in CNNs}\label{app:multi_eigenvalues}

When training CNNs, after the sharpness reaches the EoS threshold, we observe the oscillations becoming chaotic after an initial phase of dampening, and that this effect is delayed when training with higher weight decay. We show in figure \ref{fig:blank2} that this behavior is due to subsequent Hessian eigenvalues reaching the $\frac{2}{\eta}$ boundary. The slowdown in progressive sharpening also applies to additional eigenvalues, which significantly delays this effect when training with higher $\gamma$. 

\begin{figure}[h!]
    \centering
    \includegraphics[width=0.8\linewidth]{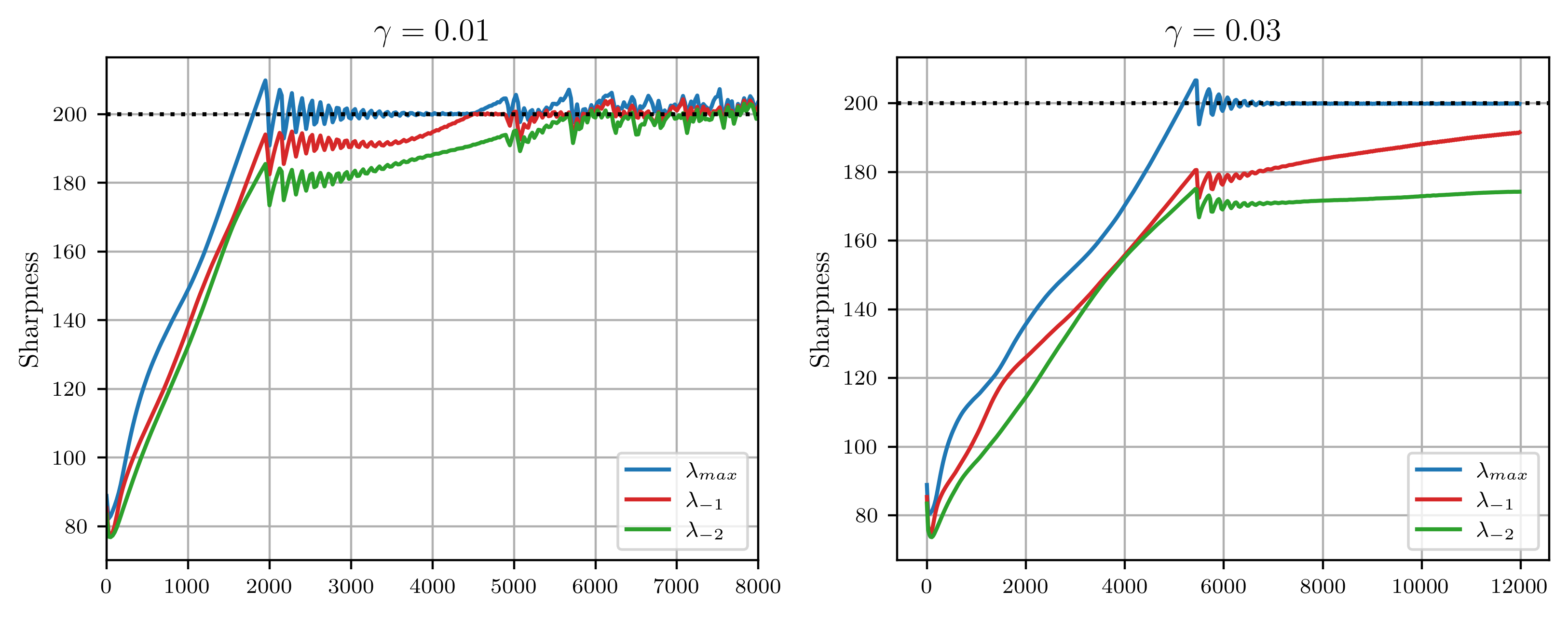}
    \caption{Evolution of eigenvalues of the loss Hessian. CNN with ReLU trained on a 5k subset of \texttt{cifar10}. When additional eigenvalues reache $2/\eta$ the sharpness $\lambda_{\operatorname{max}}$ starts behaving chaotically. This behaviour is delayed for increased $\gamma$-values, as the slowdown in progressive sharpening also applies to additional eigenvalues.}
    \label{fig:blank2}
\end{figure}

\subsection{Empirical investigation of global interaction terms}
We empirically investigate the mechanism behind the reduced EoS threshold observed in MLPs with ReLU activation. The theoretical analysis from Appendix \ref{app:model} shows that the EoS limit cycle collapses when $c_y=\nabla S\cdot \theta^*$ exceeds a critical threshold $c_y^{\operatorname{crit}}=\frac{\alpha}{\gamma}+\gamma$. In the theoretical analysis both $c_y$ and $\alpha$ are given from the fixed reference point $\theta^*$. In our empirical analysis we compute both quantities dynamically along the training trajectory, yielding the time varying $c_y(t)=\nabla S(t)\cdot \theta(t)$ and $\alpha(t)=-\nabla L(t)\cdot\nabla S(t)$. We find that the crossing condition $c_y(t)>c_y^{\operatorname{crit}}(t)$ empirically coincides with lowered EoS threshold of MLPs, and that CNNs do not cross this threshold during progressive sharpening.\footnote{We note that during early stages of training, before the onset of progressive sharpening, $\alpha(t)$ is transiently negative, placing the system in the regime $c_y>c_{y}^{\operatorname{crit}}$. This does not appear to affect the subsequent training dynamics. This is consistent with the theoretical analysis assuming $\alpha>0$.}.  

Figure \ref{fig:MLP_cycrit} illustrates the condition $c_y>c_{y}^{\operatorname{crit}}$ in an MLP for two values of $\gamma$. For small $\gamma$ (left) $c_y^{crit}$ is orders of magnitude larger than $c_y$ during progressive sharpening. It collapses when sharpness reaches $\frac{2}{\eta}$, as $\alpha$ is forced negative by the EoS-dynamics. Thereby the crossing condition coincides with standard EoS. For large $\gamma$ (right), $c_y^{crit}$ is much smaller due to the $\frac{1}{\gamma}$ scaling, allowing $c_y(t)$ to overtake it before sharpness reaches $\frac{2}{\eta}$. The sharpness then stabilizes far below the standard EoS-threshold. 

\begin{figure}[htbp]
    \centering
    \includegraphics[width=0.8\linewidth]{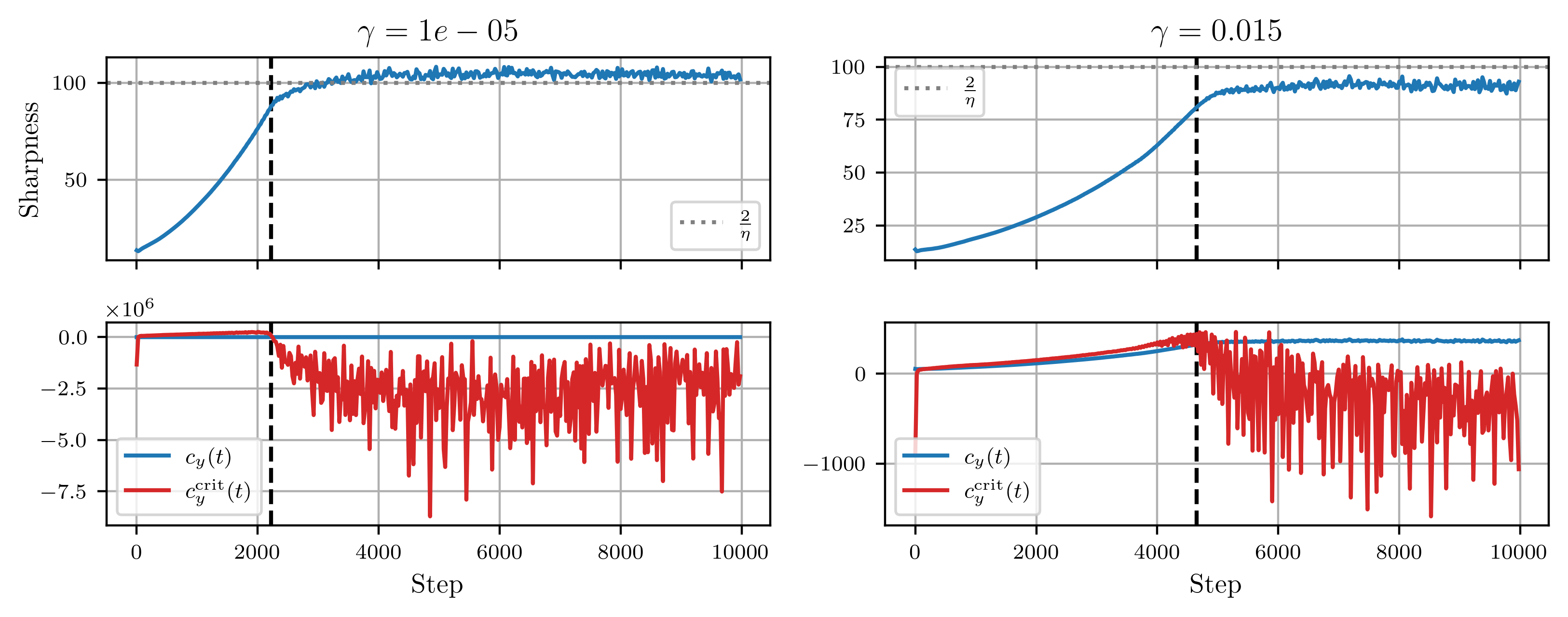}
    \caption{Sharpness (top) and $c_y(t)$, $c_y^{\operatorname{crit}}$ (bottom) during training of an MLP with ReLU on a 5k subset of \texttt{cifar10}, with $\eta=0.02$ and full batch gradient descent. For small $\gamma$ (left), $c_y^{\operatorname{crit}}$ is larger than $c_y$ until the sharpness reaches $2/\eta-\gamma$. For large $\gamma$ (right), the $\frac{1}{\gamma}$ scaling keeps $c_y^{\operatorname{crit}}$ small, allowing $c_y(t)$ to cross it before the sharpness reaches $2/\eta-\gamma$}
    \label{fig:MLP_cycrit}
\end{figure}

Figure \ref{fig:CNN_cycrit} shows the same quantities for a CNN. In contrast to the MLP, $c_y^{\operatorname{crit}}(t)$ stays above $c_y(t)$ during progressive sharpening for all tested values of $\gamma$, and the sharpness consistently reaches $\frac{2}{\eta}$. This suggests that the global offset term $c_y$ and the progressive sharpening coefficient $\alpha$ depend on the architecture in a way that prevents the crossing condition to be met during progressive sharpening in a CNN. We leave understanding this architectural dependence for future work.  

\begin{figure}[htbp]
    \centering
    \includegraphics[width=0.8\linewidth]{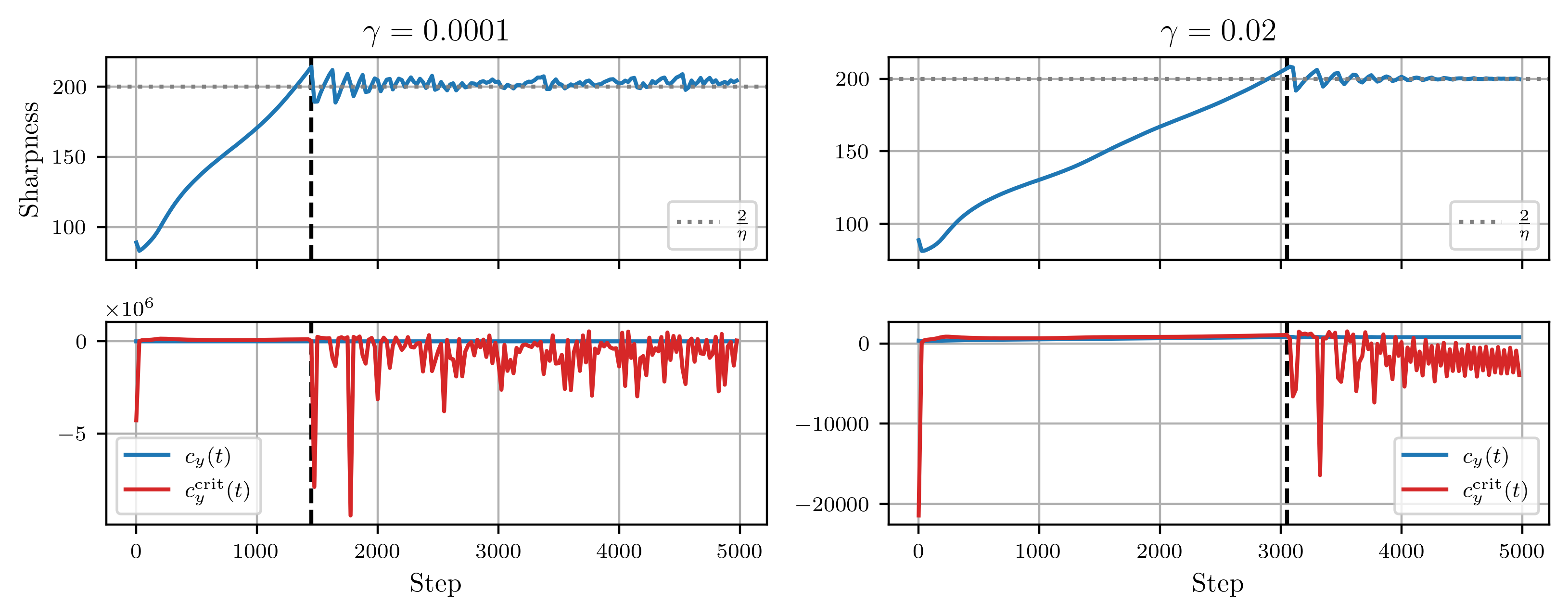}
    \caption{Sharpness (top) and $c_y(t)$, $c_y^{\operatorname{crit}}$ (bottom) during training of an CNN with ReLU on a 5k subset of \texttt{cifar10}, with $\eta=0.02$ and full batch gradient descent. For both values of $\gamma$, $c_y$ stays below $c_y^{\operatorname{crit}}$ until the sharpness reaches $2/\eta$}
    \label{fig:CNN_cycrit}
\end{figure}
\clearpage
\subsection{NTK}

The empirical NTK is defined as $\Theta=JJ^T\in\mathbb{R}^{N\times N}$, where $J\in \mathbb{R}^{N\times P}$ is the Jacobian of the network outputs with respect to all $P$ parameters, evaluated over a subset of N test points. For each sample $x_i$, we compute $\nabla_{\theta}f(x_i)\in\mathbb{R}^P$ via a single backwards pass, giving the normalized Gram matrix
\begin{equation*}
    \hat{\Theta}=\frac{1}{N}JJ^T
\end{equation*}
where the $\frac{1}{N}$ normalization keeps the Gram matrix in the same convention as the mean loss Hessian. The Jacobian was evaluated from the model weights after the sharpness had reached a stable oscillatory regime, and remained there for 2000 steps, ensuring the weights reflected the EoS regime. We extracted $\lambda_{max}(\hat{\Theta})$ via exact eigendecomposition, holding the random seed and N fixed across all $\gamma$ values. 

\paragraph{Choice of N} Since $\hat{\Theta}$ is estimated from a finite sample, its spectrum depends on N. We verify convergence by plotting $\lambda_{max}(\hat{\Theta})/N$ against N for the smallest $\gamma$. We find that the normalized eigenvalue is approximately stable by $N=5000$, and adopt this value for all experiments. 

\begin{figure}[htbp]
    \centering
    \includegraphics[width=0.6\linewidth]{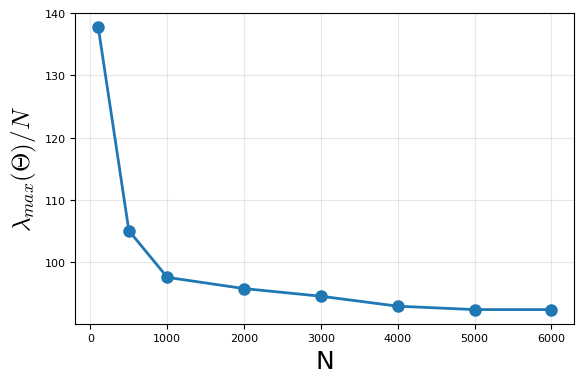}
    \caption{Convergence of the normalized NTK top eigenvalue $\lambda_{max}(\hat{\Theta})/N$ as a function of sample size N, evaluated on the smallest $\gamma$ model. The eigenvalue approximately stabilizes by N=5000}
    \label{fig:placeholder}
\end{figure}
\clearpage
\subsection{Variance across random seeds}\label{seeds}
To verify that our results are not sensitive to the choice of random seed, we repeat our MLP sharpness experiment across 6 random seeds and report the mean sharpness with $\pm2$ standard deviation in Figure \ref{fig:ex_variance}. The sharpness trajectory is consistent across seeds, suggesting that the observed phenomenon of sharpness stabilizing far below $2/\eta-\gamma$ is not an artifact of a particular initialization.
\begin{figure}[htbp]
    \centering
    \includegraphics[width=0.8\linewidth]{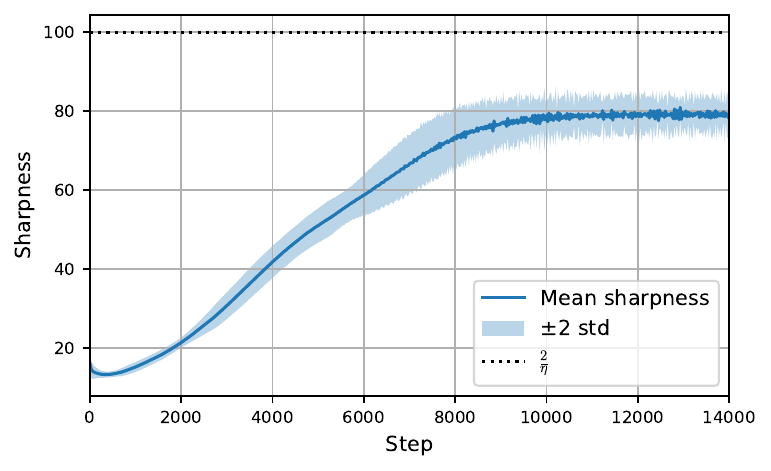}
    \caption{MLP with MSE loss trained with full batch gradient descent, $\eta=0.02$ and $\gamma=0.02$ on a 5k subset of \texttt{cifar10}, across 6 different seeds}
    \label{fig:ex_variance}
\end{figure}



\newpage

\section{Derivation of the weight decay EoS model}
\label{app:model}

Here we derive exactly the model for the dynamics at the EoS under weight decay by expanding on the framework introduced by \citep{damian_selfstab}. They show that the third-order term $\nabla^3L$ implicitly stabilizes the oscillations in the direction of maximal sharpness through the identity $\nabla^3L(u,u)=\nabla S$. By incorporating weight decay in the derivation, we show that the system can be approximated by an underdamped harmonic oscillator. Importantly, in contrast to the local dynamics in standard EoS analysis, weight decay introduces global interaction factors. The modified gradient descent update rule is given by:
\begin{equation}
\theta_{t+1} = \theta_t - \eta \nabla L(\theta_t) - \eta \gamma \theta_t = (1 - \eta\gamma)\theta_t - \eta \nabla L(\theta_t)
\end{equation}

\subsection{Setup and Structural Assumptions}
\label{app:model_assumptions}

Let $\theta^*$ be a reference point where the sharpness $S(\theta^*) = \lambda_{\max}(\nabla^2 L(\theta^*))$ is precisely at the stability threshold $\frac{2}{\eta}$.  We track the displacement $\delta_t = \theta_t - \theta^*$ around $\theta^*$. The dynamics are projected onto the top unit eigenvector $u(\theta^*)$ ($\nabla^2 L(\theta^*)u=S(\theta^*)u(\theta^*)$) and the gradient of the sharpness $\nabla S(\theta^*)$, at the fixed reference point $\theta^*$, we drop $\theta^*$ where applicable to simplify notation. We define our local coordinates as:
\begin{equation}
x_t = \langle u, \delta_t \rangle \quad \text{and} \quad y_t = \langle \nabla S, \delta_t \rangle
\end{equation}

To ensure the dynamics cleanly decouple into this $\operatorname{span}\{u, \nabla S\}$ basis, we make the following assumptions at $\theta^*$, adapted from the self-stabilization framework:

\begin{assumption}
$L \in \mathcal{C}^3$ near $\theta^*$, and the top eigenvalue $S(\theta^*)$ is unique and strictly isolated from the rest of the spectrum. In particular, $\exists\varepsilon>0$ such that the second largest eigenvalue satisfies $\lambda_2(\nabla^2 L(\theta^*)) \le \frac{2}{\eta}-\varepsilon$.
\end{assumption}

\begin{assumption}
The gradient of the sharpness has no component in the unstable direction, such that $\langle \nabla S(\theta^*), u \rangle = 0$. Furthermore, the gradient and the gradient of sharpness lie in the kernel of the Hessian: $\nabla L(\theta^*), \nabla S(\theta^*) \in \ker(\nabla^2 L(\theta^*))$.
\end{assumption}

\begin{assumption}
There is a consistent force driving sharpness upward, characterized by $\alpha = -\langle \nabla L(\theta^*), \nabla S(\theta^*) \rangle > 0$. Outside the very early phases of training, this is observed across virtually all experiments, in our work and the surrounding literature.
\end{assumption}

\subsection{Derivation of the dynamical system}
\label{app:model_update_rules}

We consider $x_{t+1}$ and $y_{t+1}$ in terms of $x_t$ and $y_t$. For $y_t$, taking the inner product of the displacement update with $\nabla S$, we obtain:
\begin{align*}
    y_{t+1}&=\langle\nabla S,(\theta_{t+1}-\theta^*)\rangle\\
    &=\langle \nabla S,\left((1-\eta\gamma)\theta_{t}-\theta^*-\eta\nabla L(\theta_t)\right)\rangle\\
    &=(1-\eta\gamma)\langle\nabla S,(\theta_t-\theta^*)\rangle+\eta\gamma\langle\nabla S,\theta^*\rangle-\eta\langle\nabla S,\nabla L(\theta_t)\rangle\\
    &=(1-\eta\gamma)y_t+\eta\gamma c_y-\eta\langle\nabla S,\nabla L (\theta_t)\rangle
\end{align*}
where $c_y=\langle\nabla S,\theta^*\rangle$ is the relevant global interaction term for the phase-shift behavior.
The gradient term can be approximated by a Taylor expansion around $\theta^*$:
\begin{align*}
    \langle \nabla S, \nabla L(\theta_t) \rangle &\approx \overbrace{\langle \nabla S, \nabla L(\theta^*) \rangle}^{-\alpha} + \overbrace{\langle \nabla S, \nabla^2 L(\theta^*) \delta_t \rangle}^{\nabla S\in\ker\left(\nabla^2 L(\theta^*)\right)} + \frac{1}{2}\overbrace{\nabla^3 L(\theta^*)(\nabla S, \delta_t, \delta_t)}^{\nabla^3L(u,u,\cdot)=\nabla S}\\
    &\approx-\alpha+|\nabla S\|^2\frac{x^2_t}{2}
\end{align*}
where we approximate the displacement by its dominant unstable component, $\delta_t \approx x_t u$. By the self-stabilizing identity (Lemma 2 in \citep{damian_selfstab}), $\nabla^3 L(u, u, \cdot) = \nabla S$. Therefore, last term evaluates to $\frac{1}{2} x_t^2 \langle \nabla S, \nabla S \rangle = \frac{1}{2} \|\nabla S\|^2 x_t^2$.

Substituting these terms back into the update, and defining $\beta = \|\nabla S\|^2$, yields the final discrete update for $y_{t+1}$:
\begin{equation} \label{eq:y_update_wd}
    y_{t+1} = (1 - \eta\gamma)y_t - \eta\gamma c_y + \eta\alpha - \eta\beta\frac{x_t^2}{2}
\end{equation}

Similarly for $x_t$, we project the displacement update onto the unstable eigenvector $u$:
\begin{align}
    x_{t+1} = (1 - \eta\gamma)x_t - \eta\gamma c_x - \eta \langle u, \nabla L(\theta_t) \rangle
\end{align}
where $c_x=\langle u,\theta^*\rangle$. We expand the gradient projection $\langle u, \nabla L(\theta_t) \rangle$:
\begin{align*}
    \langle u, \nabla L(\theta_t) \rangle \approx \langle u, \nabla L(\theta^*) \rangle + \langle u, \nabla^2 L(\theta^*) \delta_t \rangle + \frac{1}{2} \nabla^3 L(\theta^*)(u, \delta_t, \delta_t)
\end{align*}
Note that $\nabla L \in \ker(\nabla^2 L)$ and $u$ has a strictly positive corresponding eigenvalue, which implies that $\nabla L \perp u$ by the Spectral theorem. Additionally, since $u$ is the top eigenvector, $\langle u, \nabla^2 L(\theta^*) \delta_t \rangle = S(\theta^*) \langle u, \delta_t \rangle = S(\theta^*) x_t$.

To approximate the third term, we decompose $\delta_t = x_t u + \delta_\perp$. Then $\nabla^3 L(u, u, \cdot) = \nabla S$ gives:
\begin{align*}
    \nabla^3 L(u, \delta_t, \delta_t) &= x_t^2 \nabla^3 L(u, u, u) + 2x_t \nabla^3 L(u, u, \delta_\perp) + \nabla^3 L(u, \delta_\perp, \delta_\perp) \\
    &\approx x_t^2 \langle \nabla S, u \rangle + 2x_t \langle \nabla S, \delta_\perp \rangle
\end{align*}
Since $\nabla S \perp u$ we have $\nabla^3L(u,\delta_t,\delta_t)\approx2 x_t\langle\nabla S,\delta_{\perp}\rangle = 2 x_t y_t$. The update rule for $x$ is then
\begin{equation}
    x_{t+1}=(1-\eta\gamma)x_t-\eta\gamma c_x-\eta\overbrace{\left(\frac{2}{\eta}+y_t\right)}^{\approx S(\theta_t)}x_t=-x_t(1+\eta\gamma+\eta y_t)-\eta\gamma c_x
\end{equation}

\subsection{Analysis of the Dynamical System}
\label{app:model_analysis}

We analyze the discrete system derived in Section \ref{app:model_update_rules}:
\begin{align*}
    x_{t+1} &= -x_t(1 + \eta y_t + \eta\gamma) - \eta\gamma c_x \\
    y_{t+1} &= y_t(1 - \eta\gamma) + \eta\!\left(\alpha - \beta\frac{x_t^2}{2}\right) - \eta\gamma c_y
\end{align*}
with global interaction terms $c_x = \langle u, \theta^*\rangle$ and $c_y = \langle \nabla S, \theta^*\rangle$. We look for a steady period-2 oscillation of the form $x_t = (-1)^t \Delta + M$, where $\Delta > 0$ is the amplitude and $M$ is a constant offset, and assume $y_t$ has reached its resting value $y^*$. Substituting into the $x_t$ update rule:
\[
    (-1)^{t+1}\Delta + M
    = -\bigl[(-1)^t \Delta + M\bigr](1 + \eta y^* + \eta\gamma) - \eta\gamma c_x.
\]
Since the equation must hold for all $t \in \mathbb{N}$, the coefficients of each must vanish independently. Matching oscillating terms proportional to $(-1)^t \Delta$ yields:
\[
    -\Delta = -\Delta(1 + \eta y^* + \eta\gamma)
    \implies
    y^* = -\gamma.
\]
Matching constant terms gives:
\[
    M = -M(1 + \eta y^* + \eta\gamma) - \eta\gamma c_x
    = -M \cdot 1 - \eta\gamma c_x
    \implies
    M = -\frac{\eta\gamma c_x}{2}.
\]
Thus $c_x$ introduces a microscopic asymmetry in the bounce center, proportional to $\eta$, but has \emph{no effect} on the resting sharpness $y^*$. In particular, regardless of $c_x$, the stabilization threshold is
\[
    S(\theta_t) \to \frac{2}{\eta} - \gamma.
\]
We determine the resting amplitude $\Delta$ by averaging the $x_t$ update over a full two-step cycle. Since $x_t$ alternates between $M + \Delta$ and $M - \Delta$, we have $\langle x_t^2 \rangle = \Delta^2 + M^2 \approx \Delta^2$ (as $M = O(\eta)$). Setting $y_{t+1} = y_t = y^* = -\gamma$:
\[
    -\gamma = -\gamma(1-\eta\gamma) + \eta\alpha - \frac{\eta\beta}{2}\Delta^2 - \eta\gamma c_y.
\]
Dividing by $\eta$ and solving:
\[
    0 = \gamma^2 + \alpha - \frac{\beta}{2}\Delta^2 - \gamma c_y
    \implies
    \Delta = \sqrt{\frac{2(\alpha + \gamma^2 - \gamma c_y)}{\beta}}.
\]
Since a physical amplitude requires $\Delta^2 \geq 0$, the system sustains oscillations only when
\[
    \gamma c_y \leq \alpha + \gamma^2
    \iff
    c_y \leq \frac{\alpha}{\gamma} + \gamma =: c_y^{\mathrm{crit}}.
\]
The iterate $x_t$ oscillates with period 2, so we track its envelope $X(\tau) = |x_{\lfloor \tau/\eta \rfloor}|$ and sharpness deviation $Y(\tau) = y_{\lfloor \tau/\eta \rfloor}$, both as functions of continuous time $\tau = \eta t$. From the steady-state analysis, $c_x$ shifts the oscillation center by $M = -\frac{\eta\gamma c_x}{2} = O(\eta)$ but does not affect $\Delta$ or $y^*$; its contribution to the envelope dynamics is therefore sub-leading in $\eta$ and we set $c_x = 0$ here. Similarly, $c_y$ shifts the resting amplitude and is treated separately in the phase-transition analysis. Below we set $c_y = 0$ for the linearized system approximation analysis.

Under these simplifications, substituting $x_t = (-1)^t X$ and $y_t = Y$ into the update rules and passing to leading order in $\eta$ yields the envelope ODE system:
\begin{align*}
    \frac{dX}{d\tau} &= X(Y + \gamma), \qquad X \geq 0,\\
    \frac{dY}{d\tau} &= \alpha - \frac{\beta}{2}X^2 - \gamma Y.
\end{align*}
The unique positive fixed point is
\[
    X^* = \sqrt{\frac{2(\alpha+\gamma^2)}{\beta}}, \qquad Y^* = -\gamma,
\]
which corresponds to the period-2 limit cycle of the discrete system with amplitude $\Delta = X^*$ (when $c_y = 0$) and resting sharpness $y^* = -\gamma$, i.e.\ $S(\theta_t) \to \frac{2}{\eta} - \gamma$.

Linearizing around the fixed point with $U = X - X^*$, $V = Y - Y^*$, the Jacobian is:
\[
    J = \begin{pmatrix} Y^* + \gamma & X^* \\ -\beta X^* & -\gamma \end{pmatrix}
      = \begin{pmatrix} 0 & X^* \\ -\beta X^* & -\gamma \end{pmatrix}.
\]
The characteristic polynomial $\det(rI - J) = 0$ gives $r^2 + \gamma r + \beta(X^*)^2 = 0$. Substituting $\beta(X^*)^2 = 2(\alpha + \gamma^2)$:
\[
    r^2 + \gamma r + 2(\alpha + \gamma^2) = 0
    \implies
    r = -\frac{\gamma}{2} \pm i\frac{\sqrt{8\alpha + 7\gamma^2}}{2}.
\]
Since $\alpha > 0$ and $\gamma > 0$, both the real part $-\frac{\gamma}{2} < 0$ and the non-zero imaginary part are guaranteed. The sharpness deviation $V(\tau)$ therefore evolves as an \emph{underdamped harmonic oscillator}:
\[
    V(\tau) = Ae^{-\frac{\gamma}{2}\tau}\cos(\omega_d \tau + \varphi),
    \qquad \omega_d = \frac{\sqrt{8\alpha + 7\gamma^2}}{2}
\]
confirming that sharpness oscillations decay exponentially at rate of $\frac{\gamma}{2}$, and that larger weight decay yields faster settling. We now return to the full system with $c_y \neq 0$.

\begin{theorem}[Limit Cycle Collapse]
\label{app:model_theorem}
Let $c_y^{\mathrm{crit}} = \frac{\alpha}{\gamma} + \gamma$. If $c_y < c_y^{\mathrm{crit}}$, the system sustains a period-2 limit cycle with amplitude $\Delta > 0$ and the sharpness stabilizes at $S(\theta_t) \to \frac{2}{\eta} - \gamma$. If $c_y \geq c_y^{\mathrm{crit}}$, the limit cycle collapses ($\Delta = 0$), and the $-\frac{\beta}{2}x_t^2$ stabilization term vanishes. The sharpness then detaches from the EoS boundary and settles at the lower fixed point
\[
    y^* = \frac{\alpha}{\gamma} - c_y < -\gamma,
    \quad\text{i.e.,}\quad
    S(\theta_t) \ll \frac{2}{\eta} - \gamma.
\]
\end{theorem}
\begin{proof}
From the amplitude formula $\Delta^2 = 2(\alpha + \gamma^2 - \gamma c_y)/\beta$, the condition $\Delta^2 \geq 0$ fails precisely when $c_y \geq c_y^{\mathrm{crit}}$. At collapse, the $-\frac{\beta}{2}\Delta^2$ term in the averaged update for $y_t$ disappears, and setting $y_{t+1} = y_t = y^*$ in the $y_t$ update rule with $\Delta = 0$ yields:
\[
    y^* = y^*(1-\eta\gamma) + \eta\alpha - \eta\gamma c_y
    \implies
    \eta\gamma y^* = \eta(\alpha - \gamma c_y)
    \implies
    y^* = \frac{\alpha}{\gamma} - c_y.
\]
Since $c_y > c_y^{\mathrm{crit}} = \frac{\alpha}{\gamma} + \gamma$ implies $\frac{\alpha}{\gamma} - c_y < -\gamma$, the resting sharpness falls strictly below the standard shifted threshold $-\gamma$, completing the proof.
\end{proof}

In practice, both $c_y$ and $\alpha$ evolve along the training trajectory. The phase transition occurs during training if $c_y(t)$ crosses $c_y^{\mathrm{crit}}(t) = \frac{\alpha(t)}{\gamma} + \gamma$ before the sharpness reaches $\frac{2}{\eta}$, as empirically observed in MLPs (Figure \ref{fig:empirical_combined}) but not in CNNs.

\newpage

\section{NTK Discussion and Proof of Lemma 5.1}
\label{app:ntk_connection}

In this section we consider the work of \citep{ntk_eos} and prove Lemma \ref{lemma:ntk_connection}. We restate the lemma for convenience before providing its formal proof.

\begin{lemma}[Lemma 5.1, restated]
Let $\alpha_t := \frac{\|v_t\|^2} {c_t^2}$ be the ratio between the output norm component and the sharpness component of the NTK. Under mild regularity conditions on the residual, a reduction in the sharpness threshold yields an overall increase in $\alpha_t$. Furthermore, for an appropriate range, increasing $\alpha_t$ shifts the maximum alignment of the target vector $Y$ towards the leading eigenvectors of the NTK.
\end{lemma}

\subsection{Setup and Structural Assumptions}

We work within the two-layer linear network framework of \citep{ntk_eos}, which we briefly recall. Concretely, we consider a network with output $F_t = W_t^{(2)} W_t^{(1)} X = c_t v_t$, where the rank-$1$ structure $W_t^{(1)} X \approx u v_t^T$, $W_t^{(2)} \approx c_t u^T$ is preserved throughout training (Assumption 1 and Appendix B.2 of \citep{ntk_eos}). Under this structure, the empirical NTK decomposes as
\[
    K_t = c_t^2 X^TX + v_t v_t^T.
\]
Let $\hat{y} = \frac{Y^T} {\|Y\|}$ be the unit target vector. At times $t_1$ (start of Phase III) and $t_2$ (end of Phase IV) in the EoS oscillation, Theorem 1(B) of \citep{ntk_eos} establishes that $\cos(v_{t_i}, Y) \geq 1 - \mathcal{O}(\delta_2)$ for small $\delta_2$, justifying the approximation $v_t \approx \|v_t\|\hat{y}$ at both endpoints. Substituting this into the equation above gives the rank-$1$ NTK form used in the main text:
\begin{equation}
    K_t \approx c_t^2 X^T X + \|v_t\|^2\hat{y}\hat{y}^T.
    \label{eq:ntk-rank1}
\end{equation}
We define $\alpha_t := \frac{\|v_t\|^2}{c_t^2}$ throughout.

\subsection{Effect of a Reduced EoS Threshold}

The main claim of the first part of Lemma \ref{lemma:ntk_connection} is that the phase transition caused by weight decay depresses the sharpness stabilization level well below the naive $\frac{2}{\eta} - \gamma$ bound and causes $\alpha_t$ to increase.

The sharpness in the two-layer linear setting satisfies $\lambda_{\max}(\nabla^2 L) \approx \frac{\lambda_1}{n} c_t^2$ (see Equation 2 of \citep{ntk_eos} and the surrounding discussion), where $\lambda_1$ is the leading eigenvalue of $X^T X$ and $n$ is the sample size. Therefore, a reduced EoS sharpness threshold $S^*$ implies a reduced steady-state value of $c_t^2$:
\[
    S^* \approx \frac{\lambda_1} {n}\left(c^*_t\right)^{2}
    \Longrightarrow
    \left(c^*_t\right)^{2} = \frac{n S^*}{\lambda_1}.
\]
Crucially, the weight decay phase transition depresses $S^*$ (Theorem \ref{app:model_theorem} and Section \ref{sec:model}), so $\left(c^*_t\right)^{2}$ is correspondingly reduced.

We now argue that $\|v_t\|^2$ does \emph{not} decrease such that it offsets the decrease in $c^2_t$. Theorem 1(A) of \citep{ntk_eos} shows that throughout Phase III, $\|v_{t+1}\|^2 > \|v_t\|^2$ at each step, driven by the $\eta^2$ correction term $\frac{\Delta_t \eta^2}{n}\lambda_1 \langle E_t, q_1\rangle^2$ which is \emph{positive} whenever $\Delta_t > 0$ (i.e., whenever $\frac{\lambda_1}{n}c_t^2 > \frac{2}{\eta}$). Moreover, Theorem 1(B) provides an \emph{overall} increase across Phases III and IV. Under Assumption 4 ($\|E_{t_2}\|^2 \leq \|E_{t_1}\|^2$) and the condition $\Delta_{t_1} \geq \Omega(\frac{\delta_2}{\eta})$, one obtains $\alpha_{t_2} > \alpha_{t_1}$.

Under weight decay, the phase transition forces the sharpness to stabilize at a threshold $S^*$ that is substantially lower than the baseline EoS level. This depresses $\left(c^*_t\right)^{2}$. Since the long-run behavior of $\|v_t\|^2$ is monotonically increasing (see the Phase I analysis in Appendix B.3 of \cite{ntk_eos} and the per-phase arguments above), while $c_t^2$ is pinned to a reduced oscillation center, the ratio
\[
    \alpha_t = \frac{\|v_t\|^2}{c_t^2}
\]
is strictly inflated by the reduction in $c_t^{2,*}$ induced by the lower sharpness threshold. This completes the argument for the first clause of Lemma \ref{lemma:ntk_connection}.

Note that the argument above is not merely correlational. The weight decay phase transition \emph{mechanically} suppresses $c_t^2$ via the oscillator collapse characterized in Theorem \ref{app:model_theorem} (Appendix \ref{app:model_analysis}), while the increasing trend of $\|v_t\|^2$ is an \emph{independent} consequence of the $\eta^2$-order correction that dominates Phase III dynamics. The two effects are decoupled, and their combination yields the inflation of $\alpha_t$.

\subsection{Resulting Shifts of Target Alignment to Leading Eigenvectors}

The second claim of Lemma 5.1 concerns the effect of an increased $\alpha_t$ on the alignment
between $Y$ and the leading eigenvectors of $K_t$.

Let the data kernel $X^T X$ be decomposed as $X^T X = Q\mathrm{diag}(\lambda_1, \ldots, \lambda_n)Q^T$, with $\lambda_1 \geq \cdots \geq \lambda_n \geq 0$, and let $b = Q^T \hat{y}$. Under the approximate rank-1 NTK structure \eqref{eq:ntk-rank1}, the NTK matrix (up to a factor of $c_t^2$) takes the form
\begin{equation}
    A(\alpha) := \mathrm{diag}(\lambda_1, \ldots, \lambda_n) + \alpha b b^T,
    \label{eq:A-alpha}
\end{equation}
where $\alpha = \alpha_t$ and $\|b\|^2 = \|\hat{y}\|^2 = 1$.

Lemma 3 of \citep{ntk_eos} characterizes the alignment structure of $A(\alpha)$ rigorously. Under the data distribution described in Section 4.1 of the same paper, in which the first $k$ eigenvalues are mutually separated by a factor of $\gamma > 1$, and the remaining $n-k$ eigenvalues are of order $\Theta(n^a)$ with $a \in (0,1)$, we have that for any $j \leq k-1$, if
\begin{equation}
    \frac{\lambda_j - \lambda_k}{1 - \mathcal{O}(\delta_\lambda + n^{\frac{-a}{2}})}
    <
    \alpha
    <
    \frac{\lambda_{j-1} - \lambda_k}{1 + \mathcal{O}(\delta_\lambda + n^{\frac{-a}{2}})},
    \label{eq:alpha-band}
\end{equation}
then the $j$-th eigenvector $\tilde{q}_j$ of $A(\alpha)$ is the one most aligned with $\hat{y}$,
i.e.,
\begin{equation}
    \operatornamewithlimits{argmax}_{1 \leq i \leq n}
    |\langle \tilde{q}_i, \hat{y} \rangle|
    = j.
    \label{eq:argmax-align}
\end{equation}
The mechanism driving this is the rank-one perturbation structure of $A(\alpha)$. By the Bunch-Nielsen-Sorensen formula \citep{bunch}, the eigenvalues of $A(\alpha)$ satisfy a secular equation, and as $\alpha$ increases, the $j$-th eigenvalue $\tilde\lambda_j$ is pushed progressively upward past successive gaps in the spectrum of $X^T X$. At each crossing $\tilde\lambda_j \approx \lambda_{j-1} - \lambda_k$, the corresponding eigenvector rotates to align more strongly with $b = Q^T \hat{y}$, i.e., with the target.

Since larger $\alpha$ corresponds to larger values of $j-1$ satisfying \eqref{eq:alpha-band}, increasing $\alpha_t$ shifts the index achieving \eqref{eq:argmax-align} to a \emph{smaller} value of $j$ (i.e., a \emph{more leading} eigenvector). This is the alignment shift phenomenon.
Thus:
\begin{equation}
    \alpha_{t_2} > \alpha_{t_1}
    \Longrightarrow
    \text{$Y$ aligns more strongly with earlier eigenvectors of $K_{t_2}$
          than of $K_{t_1}$.}
    \label{eq:alignment-shift}
\end{equation}

\subsection{Proof of Lemma 5.1}

\begin{proof}
\textbf{Part 1: Reduced sharpness threshold increases $\alpha_t$.}
By the two-layer linear network analysis of \citep{ntk_eos}, the sharpness satisfies $\lambda_{\max}(\nabla^2L) \approx \frac{\lambda_1}{n}c_t^2$ at the EoS, so the reduced stabilizing sharpness $S^*$ induced by the weight decay phase transition (Theorem \ref{app:model_theorem}) implies $\left(c^*_t\right)^{2} = \frac{nS^*}{\lambda_1}$, which is strictly smaller than the baseline value $c^2_t = \frac{2n}{(\eta\lambda_1)}$. Theorem 1(A) and (B) of \citep{ntk_eos} establish that $\|v_t\|^2$ increases monotonically across Phases I to III and has a net increase across the full Phase III to IV oscillation. Since $\left(c^*_t\right)^{2}$ is depressed while $\|v_t\|^2$ is non-decreasing, the ratio $\alpha_t = \frac{\|v_t\|^2}{c_t^2}$ is strictly larger under the reduced threshold. \medskip

\textbf{Part 2: Increased $\alpha_t$ shifts alignment to leading eigenvectors.} Lemma 3 of \citep{ntk_eos} establishes \eqref{eq:argmax-align} for each window of $\alpha$ satisfying \eqref{eq:alpha-band}. As $\alpha$ increases, the maximally aligned eigenvector index $j$ decreases (i.e., shifts to a more leading eigenvector). Therefore, $\alpha_{t_2} > \alpha_{t_1}$ implies the alignment described in \eqref{eq:alignment-shift}, completing the proof of both clauses.
\end{proof}

\subsection{Connection to the Main Argument of Section 5}

Lemma \ref{lemma:ntk_connection} provides the mechanism linking the parameter-space phenomenon (weight decay depressing
the EoS sharpness threshold) to the function-space consequence (improved NTK--target alignment).
The chain of reasoning is as follows.

\begin{enumerate}[label=(\roman*)]
    \item Weight decay induces a phase transition (Theorem \ref{app:model_theorem}) that pins the sharpness oscillation center at $S^* \ll \frac{2}{\eta} - \gamma$.
    \item This depresses $\left(c^*_t\right)^{2}$.
    \item By Theorem 1 of \citep{ntk_eos}, $\|v_t\|^2$ continues to increase throughout the EoS oscillation regardless of the sharpness level.
    \item Therefore $\alpha_t = \frac{\|v_t\|^2}{c_t^2}$ is inflated (Part 1 of Lemma \ref{lemma:ntk_connection}).
    \item By Lemma 3 of \citep{ntk_eos}, increased $\alpha_t$ shifts the leading NTK
          eigenvectors to align more strongly with $Y$ (Part 2 of Lemma \ref{lemma:ntk_connection}).
    \item Combined with the Weyl bound \eqref{eq:ntk_eos_bound} given in Section \ref{sec:ntk} (which ties the Hessian's EoS threshold to the NTK spectral radius as the residual $\|R\|_2 \to 0$), the restricted NTK spectrum and the improved alignment together yield the improved kernel target alignment (KTA) and feature learning quality documented empirically in Figure \ref{fig:ntk_eigenvalue_decrease}.
\end{enumerate}

We note that the precise conditions of Lemma 3 of \citep{ntk_eos}, the gap condition on $\{(\lambda_j - \lambda_k)\}$ and the evenly-spread assumption on $b$, hold in the data regime of Section 4.1 of that work. While our empirical evaluation uses full \texttt{cifar10} subsets with realistic data geometry, the lemma serves as a rigorous heuristic establishing the \emph{direction} of the alignment shift; quantitative predictions in richer data regimes are left for future work, as noted in Section \ref{sec:limitations} of the main text.

\section{Experimental Details}

\subsection{Data}
We trained on a subset of \texttt{cifar10} consisting of the first 5000 training examples in their canonical ordering. The resulting class distribution is approximately uniform, with per-class count ranging from 460-520 (ideal:500). Input images are normalized by subtracting the per-channel mean and dividing by per-channel standard deviation, both computed over the full 50 000 samples.  

\subsection{Architectures}
We conducted experiments on two architectures. The first is an MLP, with two hidden layers, ReLU activations, and a linear output layer of 10 units ($\sim 657k$ parameters). The second is a CNN with two convolutional layers of 32 filters each, kernel size $3\times 3$, each followed by ReLU activations, and $2\times2$ max pooling, and a final linear classifier ($\sim 31k$ parameters). Both architectures use bias terms. 

\subsection{Main experimental Setup}
For a fixed architecture with a fixed step size $\eta$, we trained the model on the \texttt{cifar10} subset using full batch gradient descent with mean squared error (MSE) loss, varying the weight decay parameter $\gamma$ in the range $[10^{-5},0.03]$. Training was run until the sharpness stabilized, requiring between 4000 and 25000 steps depending on the architecture and the level of weight decay (larger $\gamma$ requires more steps). The training loss was recorded at every iteration, while the sharpness ($\lambda_{max}(H)$), the progressive sharpening coefficient ($\alpha$), and the global interaction term ($c_y$) were recorded every 25 steps due to their computational cost. 

We compute sharpness as the top eigenvalue of the \emph{unregularized} training loss Hessian $\lambda_{max}(\nabla^2 L)$. Direct Hessian materialization for a model with $p$ parameters requires $\mathcal{O}(p^2)$ memory, which is infeasible even for relatively small networks. We therefore estimated the sharpness through Lanczos algorithm applied to Hessian vector products computed via second order automatic differentiation. To attain $c_y=\nabla_{\theta} \lambda_{max}(H)\cdot\theta$ and $\alpha=-\nabla_\theta L\cdot\nabla_{\theta} \lambda_{max}(H)$, we compute $\nabla_{\theta} \lambda_{max}(H)$ using the identity $\nabla_{\theta} \lambda_{max}(H)=\nabla_{\theta}(u^THu)$, where $u$ is the top eigenvector obtained from the Lanczos pass. This avoids any additional eigenvalue computation. 

All experiments were conducted on a single NVIDIA L4 GPU (24GB VRAM), with each training run having a runtime between 40 minutes and 2 hours (higher $\gamma$ required more steps). All experiments were run across multiple seeds. The qualitative findings reported are consistent across seeds (see Appendix \ref{seeds}). The main plots presented in section \ref{sec:empirical} are shown for a representative seed (seed 10). 

\subsection{Toy model}

For $\alpha>0$ and $\beta>0$, the toy model is defined by:
\begin{equation*}
    L(x,y,z)=\left( \frac{2}{\eta}+\sqrt{\beta}y \right)\frac{x^2}{2}-\frac{\alpha}{\sqrt{\beta}}y-z
\end{equation*}

This is the same toy example presented in \citep{damian_selfstab}.


\label{app:exp_details}



\end{document}